\documentclass{article}

\usepackage{microtype}
\usepackage{graphicx}
\usepackage{subfigure}
\usepackage{booktabs} 

\usepackage{amsfonts,amsthm,mathrsfs,amssymb,amsmath}

\newtheorem{thm}{Theorem}
\newtheorem{lemma}[thm]{Lemma}
\newtheorem{pro}[thm]{Proposition}
\newtheorem{corollary}[thm]{Corollary}

\newtheorem{rem}{Remark}
\newtheorem{Exa}{Example}[section]
\newtheorem{as}{Assumption}
\newtheorem{alg}{Algorithm}

\newcommand{\be}{\begin{equation}}
\newcommand{\ee}{\end{equation}}
\newcommand{\bea}{\begin{eqnarray*}}
\newcommand{\eea}{\end{eqnarray*}}
\newcommand{\bflalign}{\begin{flalign*}}
\newcommand{\eflalign}{\end{flalign*}}

\newcommand{\mR}{\mathbb{R}}

\newcommand{\mN}{\mathbb{N}}
\newcommand{\mE}{\mathbb{E}}
\newcommand{\mcE}{{\mathcal{E}}}

\newcommand{\mcS}{\mathcal{S}}

\newcommand{\mcN}{\mathcal{N}}

\newcommand{\mcX}{\mathcal{X}}
\newcommand{\mcV}{\mathcal{V}}
\newcommand{\bz}{{\bf z}} 
\newcommand{\bx}{{\bf x}} 
\newcommand{\by}{{\bf y}}

\newcommand{\bK}{{\bf K}}

\newcommand{\tr}{\operatorname{tr}}
\newcommand{\la}{\langle}
\newcommand{\ra}{\rangle}
\newcommand{\eref}[1] {(\ref{#1})}

\newcommand{\TK}{\mathcal{T}} 

\newcommand{\TKL}{\mathcal{T}_{\PRegPar}} %
\newcommand{\TXL}{\mathcal{T}_{{\bf x}\PRegPar}}

\newcommand{\TKLL}{\mathcal{T}_{\lambda}}

\newcommand{\LK}{\mathcal{L}}

\newcommand{\IK}{\mathcal{S}_{\rho}}
\newcommand{\TX}{\mathcal{T}_{\bf x}}
\newcommand{\SX}{\mathcal{S}_{\bf x}}
\newcommand{\SXS}{\mathcal{S}_{\bf \tilde{x}}} 
\newcommand{\TXS}{\mathcal{T}_{\bf \tilde{x}}}

\newcommand{\HK}{H}
\newcommand{\HR}{H_{\rho}}
\newcommand{\LR}{L^2_{\rho_X}}
\newcommand{\HKS}{S}

\newcommand{\st}{P}
\newcommand{\proj}{P}

\newcommand{\GL}{\mathcal{G}_{\lambda}} 
\newcommand{\RL}{\mathcal{R}_{\lambda}} 

\newcommand{\FR}{f_{\rho}}
\newcommand{\FH}{f_{\HK}}

\newcommand{\DZF}{\Delta_1} 
\newcommand{\DZS}{\Delta_2} %
\newcommand{\DZT}{\Delta_3} 
\newcommand{\DZN}{\Delta_4} 
\newcommand{\DZI}{\Delta_5}

\newcommand{\RegPar}{\lambda}
\newcommand{\PRegPar}{{\lambda}}

\newcommand{\ESRA}{\omega_{\lambda}^{{\bf z}}} 
\newcommand{\EPSRA}{\omega_{\lambda}} 
\newcommand{\EESRA}{f_{\lambda}^{\bz}}

\newcommand{\Outputs}{{\bf y}}

\newcommand{\mcEE}{\widetilde{\mathcal{E}}}

\newcommand{\ao}{\mathcal{U}} 
\newcommand{\aoB}{\mathcal{V}} 

\newcommand{\TA}{{\bf Term.A}} 

\newcommand{\TB}{{\bf Term.B}} 

\newcommand{\skt}{{\bf G}} 

\newcommand{\ba}{{\bf a}}

\newcommand{\mcNx}{\mathcal{N}_{\bx}(\lambda)} 

\newcommand{\aol}{\ao_{\lambda}}
\newcommand{\aoBl}{\aoB_{\lambda}}

\newcommand{\bby}{\bar{\by}}
\newcommand{\NOutputs}{\bby}

\newcommand{\mP}{\mathbb{P}}

\usepackage[accepted,nohyperref]{icml2018}

\icmltitlerunning{Optimal Rates of Sketched-regularized Algorithms for Least-Squares Regression over Hilbert Spaces}

\begin{document}
	
	\twocolumn[
	\icmltitle{Optimal Rates of Sketched-regularized Algorithms for Least-Squares Regression over Hilbert Spaces}
	
	
	
	\icmlsetsymbol{equal}{*}
	
	\begin{icmlauthorlist}
		\icmlauthor{Junhong Lin}{to}
		\icmlauthor{Volkan Cevher}{to}
	\end{icmlauthorlist}
	
	\icmlaffiliation{to}{ 
		Laboratory for Information and Inference Systems,
		\'{E}cole Polytechnique F\'{e}d\'{e}rale de Lausanne,
		Lausanne, Switzerland}
	\icmlcorrespondingauthor{Junhong Lin}{junhong.lin@epfl.ch}
	\icmlcorrespondingauthor{Volkan Cevher}{volkan.cevher@epfl.ch}
	
	\icmlkeywords{Machine Learning, ICML}
	
	\vskip 0.3in
	]
	
	
	
	\printAffiliationsAndNotice{}  
	
	\begin{abstract}
		We investigate regularized algorithms combining with projection for least-squares regression problem over a Hilbert space, covering nonparametric regression over a reproducing kernel Hilbert space. We prove convergence results with respect to variants of norms, under a capacity assumption on the hypothesis space and a regularity condition on the target function.  
		As a result, we obtain optimal rates for regularized algorithms with randomized sketches, provided that the sketch dimension  is proportional to the
		effective dimension up to a logarithmic factor. As a byproduct, we obtain similar results for Nystr\"{o}m regularized algorithms. Our results provide optimal, distribution-dependent rates  that  do not have any saturation effect for sketched/Nystr\"{o}m regularized algorithms, considering both the attainable and  non-attainable cases. 
	\end{abstract}

  \section{Introduction}\label{sec:introduction}

Let the input space $\HK$ be a separable Hilbert space with inner product denoted by
$\la \cdot, \cdot \ra_{\HK}$, and the output space $\mR$. 
Let $\rho$ be an unknown probability measure on $\HK \times \mR$. 
In this paper, we study the following expected risk minimization,
\be\label{expectedRiskA} 
\inf_{\omega \in \HK} \mcEE(\omega), \quad \mcEE(\omega) = \int_{\HK\times \mR} ( \la \omega, x \ra_{\HK} - y)^2 d\rho(x,y),
\ee
where the measure  $\rho$ is  known only through
a sample $\bz =\{z_i=(x_i, y_i)\}_{i=1}^n$ of size $n\in\mN$, independently and identically distributed (i.i.d.) according to $\rho$.

The above regression setting covers nonparametric regression over a reproducing kernel Hilbert space \cite{cucker2007learning,steinwart2008support}, and  it is close to functional regression \cite{ramsay2006functional} and linear inverse problems \cite{engl1996regularization}.  A basic algorithm for the problem is ridge regression, and its generalization, spectral algorithm. Such algorithms can be viewed as solving an empirical, linear equation with the empirical covariance operator replaced by a regularized one, see \cite{caponnetto2006adaptation,bauer2007regularization,gerfo2008spectral,lin2018optimal} and references therein.  Here, the regularization is used to control the complexity of the solution to against over-fitting and to achieve best generalization ability. 

The function/estimator generated by classic regularized algorithm is in the subspace $\overline{span\{\bx\}}$ of $\HK$, where $\bx=\{x_1,\cdots, x_n \}.$ 
More often, the search of an estimator for some specific algorithms
is restricted to a different (and possibly smaller) subspace $\HKS$, which leads to regularized algorithms with projection.  Such approaches have computational advantages in nonparametric regression with kernel methods \cite{williams2000using,smola2000sparse}. Typically, with a subsample/sketch dimension $m<n$, $\HKS = \overline{span\{ \tilde{x}_j : 1\leq j\leq m\}}$ where $\tilde{x}_j $ is chosen randomly from the input set $\bx$, or 
$\HKS = \overline{span\{ \sum_{j=1}^m G_{ij} x_j : 1\leq i\leq m\}}$ where $\skt = [G_{ij}]_{1\leq i\leq m, 1\leq j\leq n}$ is a general randomized matrix whose rows are drawn according to a distribution. 
The resulted algorithms are called Nystr\"{o}m regularized algorithm and sketched-regularized algorithm, respectively.

Our starting points of this paper are recent papers \cite{bach2013sharp,el2014fast,yang2015randomized,rudi2015less,myleiko2017regularized} where convergence results on Nystr\"{o}m/sketched regularized algorithms for learning with kernel methods are given.
Particularly, within the fixed design setting, i.e., the input set $\bx$ are deterministic while the output set $\by=\{y_1,\cdots,y_n\}$ treated randomly, convergence results have been derived, in \cite{bach2013sharp,el2014fast} for Nystr\"{o}m ridge regression and in \cite{yang2015randomized} for sketched ridge regression.  Within the random design setting (which is more meaningful \cite{hsu2014random} in statistical learning theory) and involving a regularity/smoothness condition on the target function \cite{smale2007learning}, optimal statistical results on generalization error bounds (excess risks) have been obtained in \cite{rudi2015less} for Nystr\"{o}m ridge regression.  The latter results were further generalized in \cite{myleiko2017regularized} to a general Nystr\"{o}m regularized algorithm. \\
Although results have been developed for sketched ridge regression in the fixed design setting, it is still unclear if one can get statistical results for a general sketched-regularized algorithms in the random design setting. Besides, all the derived results, either for sketched or Nystr\"{o}m regularized algorithms,  are only for the attainable case, i.e., the case that the expected risk minimization \eqref{expectedRiskA} has at least one solution in $\HK$. Moreover, they saturate \cite{bauer2007regularization} at a critical value, meaning that they can not lead to better convergence rates even with a smoother target function.
Motivated by these, in this paper, we study statistical results of projected-regularized algorithms for least-squares regression over a separable Hilbert space within the random design setting.  

We first extend the analysis in \cite{lin2018optimal} for classic-regularized algorithms to projected-regularized algorithms, and prove statistical results with respect to a broader class of norms.  We then show that optimal rates can be retained for sketched-regularized algorithms, provided that the sketch dimension is proportional to the effective dimension \cite{zhang2006learning} up to a logarithmic factor. As a byproduct, we obtain similar results for Nystr\"{o}m regularized algorithms. 

Interestingly, our results are the first ones with optimal, distribution-dependent rates that do not have any saturation effect for sketched/Nystr\"{o}m regularized algorithms, considering both the attainable and non-attainable cases.  In our proof, we naturally integrate  proof techniques from \cite{smale2007learning,caponnetto2007optimal,rudi2015less,myleiko2017regularized,lin2018optimal}. 
Our novelties lie in a new estimates on the projection error for sketched-regularized algorithms, a novel analysis to conquer the saturation effect, and a refined analysis for Nystr\"{o}m regularized algorithms, see Section \ref{sec:proof} for details.

The rest of the paper is organized as follows. Section \ref{sec:learning} introduces some auxiliary notations and projected-regularized algorithms. Section \ref{sec:conve} present assumptions and our main results, followed with simple discussions.
Finally, Section \ref{sec:proof} gives the proofs of our main results.

\section{Learning with Projected-regularized Algorithms} \label{sec:learning}
In this section, we introduce some notations as well as auxiliary operators, and present projected-regularized algorithms.
\subsection{Notations and Auxiliary Operators}
Let  $Z = \HK \times \mR$, $\rho_X(\cdot)$  the induced marginal measure on $\HK$ of $\rho$, and $\rho(\cdot | x)$ the conditional probability measure on $\mR$ with respect to $x \in \HK$ and $\rho$. 
For simplicity, we assume that the support of $\rho_X$ is compact and that there exists a constant $\kappa \in [1,\infty[$, such that
\be\label{boundedKernel} \la x,x' \ra_{\HK} \leq \kappa^2, \quad \forall x,x'\in \HK,   \rho_X\mbox{-almost every}.
\ee

Define the hypothesis space $\HR = \{f: \HK \to  \mR| \exists \omega \in \HK \mbox{ with } f(x) = \la \omega, x \ra_{\HK}, \rho_X \mbox{-almost surely}\}.$  Denote $\LR$ the Hilbert space of square integral functions from $\HK$ to $\mR$ with respect to $\rho_X$, with its norm given by $\|f\|_{\rho} = \left(\int_{\HK} |f(x)|^2 d \rho_X\right)^{1\over 2}.$ 

For a given bounded operator $L: \LR \to \HK, $ $\|L\|$ denotes the operator norm of $L$, i.e., $\|L\| = \sup_{f\in \LR, \|f\|_{\rho}=1} \|Lf\|_{\HK}$.  Let $r \in \mN_+,$ the set $\{1,\cdots,r\}$ is denoted by $[r].$
For any real number $a$, $a_+ = \max(a,0)$, $a_- = \min(0,a)$.

Let $\IK: \HK \to \LR$ be the linear map $\omega \to \la \omega, \cdot \ra_{\HK}$, which is bounded by $\kappa$ under Assumption \eref{boundedKernel}. Furthermore, we consider the adjoint operator $\IK^*: \LR \to \HK$, the covariance operator $\TK: \HK \to \HK$ given by $\TK = \IK^* \IK$, and the integeral operator $\LK : \LR \to \LR$ given by $\IK \IK^*.$ It can be easily proved that $ \IK^*g = \int_{\HK} x g(x) d\rho_X(x), $ $\LK f = \int_{\HK} f(x) \la x, \cdot \ra_{\HK} d \rho_X(x)$
and $\TK = \int_{\HK} \la \cdot , x \ra_{\HK} x d \rho_X(x).$  
Under Assumption \eqref{boundedKernel},
the operators $\TK$ and $\LK$ can be proved to be positive trace class operators (and hence compact):
\be\label{eq:TKBound}
\begin{split}
\|\LK\| = \|\TK\| \leq \tr(\TK) = \int_{\HK} \tr(x \otimes x)d\rho_X(x) \\=  \int_{\HK} \|x\|_{\HK}^2 d\rho_{X}(x) \leq \kappa^2.
\end{split}
\ee
For any $\omega \in \HK$,
it is easy to prove the following isometry property \cite{bauer2007regularization},
\be\label{isometry}
\|\IK \omega \|_{\rho} = \|\sqrt{\TK} \omega\|_{\HK}.
\ee
Moreover, according to the singular value decomposition of a compact operator, one can prove that
\be\label{eq:rho2hk}
\|\LK^{-{1\over2 }}\IK \omega\|_{\rho} \leq \|\omega\|_{\HK}.
\ee
We define the (modified) sampling operator $\SX: \HK \to \mR^n$ by $(\SX \omega)_i = {1\over \sqrt{n}}\la \omega, x_i \ra_{\HK},$ $i \in [n]$, where the norm $\|\cdot\|_{\mR^n}$ in $\mR^n$ is the usual Euclidean norm.
Its adjoint operator $\SX^*: \mR^n \to \HK,$ defined by $\la \SX^*{\bf y}, \omega \ra_{\HK} = \la {\bf y}, \SX \omega\ra_{\mR^n}$ for ${\bf y} \in \mR^n$ is thus given by $\SX^*{\bf y} = {1 \over \sqrt{n}} \sum_{i=1}^n y_i x_i.$
For notational simplicity, we let
$\bby = {1\over \sqrt{|\by|}} \by.$
 Moreover, we can define the empirical covariance operator $\TX: \HK \to \HK$ such that $\TX = \SX^* \SX$. Obviously,
$
\TX = {1 \over n} \sum_{i=1}^n \la \cdot, x_i \ra_{\HK} x_i.
$
By Assumption \eqref{boundedKernel}, similar to \eref{eq:TKBound}, we have
\be\label{eq:TXbound}
\|\TX\| \leq \tr(\TX) \leq \kappa^2.
\ee

It is easy to see that Problem \eqref{expectedRiskA} is equivalent to 
\be\label{expectedRisk}
\inf_{f \in \HR} \mcE(f), \quad \mcE(f)= \int_{\HK \times \mR} ( f(x) - y)^2 d\rho(x,y),
\ee
The function that minimizes the expected risk over all measurable functions is the
regression function \cite{cucker2007learning,steinwart2008support}, defined as,
\be\label{regressionfunc}
f_{\rho}(x) = \int_{\mR} y d \rho(y | x),\qquad x \in \HK,  \rho_X\mbox{-almost every}.
\ee
A simple calculation shows that the following  well-known fact holds \cite{cucker2007learning,steinwart2008support}, for all  $f \in \LR,$
$
\mcE(f) - \mcE(\FR) = \|f - \FR\|_{\rho}^2.
$
Then it is easy to see that \eqref{expectedRisk} is equivalent to 
$
\inf_{f\in \HR} \|f - \FR\|_{\rho}^2.
$
Under Assumption \eqref{boundedKernel},  $\HR$ is a subspace of $\LR.$
Using the projection theorem, one can prove that
a  solution $\FH$ for the problem   \eqref{expectedRisk} is the projection of the regression function $f_{\rho}$ onto the closure of $\HR$ in $\LR$, and moreover,  for all $f\in \HR$ \cite{lin2017optimal},
\be\label{frFH}
\IK^* f_{\rho} = \IK^* \FH,
\ee
and
\be\label{eq:exceRisk}
\mcE(f) -  \mcE(\FH) = \|f - \FH\|_{\rho}^2.
\ee
{\it Note that  $\FH$ does not necessarily be in $\HR$}.

Throughput this paper, $\HKS$ is a closed, finite-dimensional subspace of $\HK$, and $\proj$ is the projection operator onto $S$ or $\proj = I$. 
\subsection{Projected-regularized Algorithms} \label{sec:spec}
In this subsection, we demonstrate and introduce projected-regularized algorithms. 

The expected risk $\mcEE(\omega)$ in \eref{expectedRiskA} can not be computed exactly.  It can be only approximated through the empirical risk $\mcEE_{\bz}(\omega),$ 
$\mcEE_{\bz}(\omega) = {1\over n} \sum_{i=1}^n ( \la \omega, x_i \ra_{\HK} - y_i)^2.
$  
A first idea to deal with the problem is to replace the objective function in \eqref{expectedRiskA} with the empirical risk. Moreover,  we restrict the solution to  the subspace $\HKS$. This leads to 
the projected empirical risk minimization,
$
\inf_{\omega \in \HKS} \mcEE_{\bz}(\omega).
$
Using $\proj^2 = \proj,$ a simple calculation shows that a solution for the above is given by
$\hat{\omega} = \st \hat{\alpha}$, with $\hat{\alpha}$ satisfying
$
\st \TX \st \hat{\alpha} = \st \SX^* \NOutputs.
$
Motivated by the classic (iterated) ridge regression, we replace $\st \TX \st $ with a regularized one, and thus leads to the following projected (iterated) ridge regression.

\begin{alg}
	\label{alg:dSpe}
	The projected (iterated) ridge regression algorithm of order $\tau$ over the samples $\bz$ and the subspace $\HKS$ is given by $\EESRA = \IK \ESRA$, where \footnote{Let $L$ be a self-adjoint, compact operator over a separable Hilbert space $\HK$. $\GL(L)$ is an operator on $L$ defined by spectral calculus: suppose that $\{(\sigma_i, \psi_i)\}_i$ is a set of
		normalized eigenpairs of $L$ with the eigenfunctions $\{\psi_i\}_i$ forming an orthonormal basis
		of $\HK$, then $\GL(\LK) = \sum_i \GL(\sigma_i) \psi_i \otimes \psi_i.$}
	\be\label{eq:ESRA}
	\ESRA = \st\GL(\st \TX \st)\st \SX^* \bby,\ \  \GL(u) = \sum_{i=1}^{\tau} \lambda^{i-1} (\lambda+u)^{-i} .
	\ee
\end{alg}

\begin{rem}
	1)  Our results not only hold for projected ridge regression, but also hold for a general projected-regularized algorithm, in which  $\GL$ is a general filter function. Given $\Lambda \subset \mR_+,$
 a class of functions $\{\GL: [0, \kappa^2] \to [0,\infty[, \lambda \in \Lambda \}$ are called filter functions with qualification $\tau $ ($\tau\geq 1$) if there exist some positive constants $E,F<\infty$ such that
	\be
	\label{eq:GLproper1}
	 \sup_{\lambda\in \Lambda}\sup_{u \in ]0,\kappa^2] } |\GL(u) (u+\lambda) |\leq E.
	\ee
	and 
	\be\label{eq:GLproper4}
	\sup_{\alpha\in [0, \tau]} \sup_{\lambda \in \Lambda}	\sup_{u\in]0, \kappa^2]} |1 - \GL(u)u|(u+\lambda)^{\alpha}\lambda^{-\alpha} \leq F .\ee
	2) A simple calculation shows that
	\be\label{eq:regfucB}
		\GL(u) = {1 - q^{\tau} \over u} = {\sum_{i=0}^{\tau-1}q^{i} \over u+\lambda },\quad q ={\lambda \over \lambda+u}.
	\ee
	Thus, $\GL(u)$ is a filter function with qualification $\tau$, $E=\tau$ and $F= 1.$
	When $\tau=1$, it is a filter function for classic ridge regression and
		the algorithm is projected ridge regression. \\
	3) Another typical filter function studied in the literature is
	$
	\GL(u) = u^{-1} \mathrm{1}_{\{u\geq \lambda\}},
	$
	which corresponds  to principal component (spectral cut-off) regularization. Here, $\mathrm{1}_{\{\cdot\}}$ denotes the indication function.
	In this case, $E=2$, $F = 2^{\tau}$ and $\tau$ could be any positive number.
\end{rem}
In the above,  $\lambda$ is a regularization parameter which needs to be well chosen in order to achieve best performance. Throughout this paper, we assume that ${1/n}\leq \lambda\leq 1.$

The performance of an estimator $\EESRA$ can be measured in terms of {\it excess risk} ({\it generalization error}), 
$
\mcE(\EESRA)  - \inf_{\HR} \mcE = \mcEE(\ESRA) - \inf_{\HK}\mcEE,
$
which is exactly $ \|\EESRA - \FH\|_{\rho}^2$ according to \eqref{eq:exceRisk}.
Assuming that $\FH \in \HR$, i.e., $\FH  = \IK \omega_*$ for some $\omega_*\in \HK$ (in this case, the solution with minimal $\HK$-norm for $\FH = \IK \omega$ is denoted by $\omega_{\HK}$), it can be measured in terms of $\HK$-norm,
$
\|\ESRA - \omega_{\HK}\|_{\HK},
$ 
which is closely related to $\|\LK^{-{1\over 2 }}\IK(\ESRA - \omega_{\HK})\|_{\HK} = \|\LK^{-{1\over 2 }}(\EESRA - \FH)\|_{\rho} $, according to \eqref{eq:rho2hk}.
In what follows, we will measure the performance of an estimator $\EESRA$ in terms of a broader class of norms,
$
\|\LK^{-a}(\EESRA - \FH)\|_{\rho},
$
where $a\in [0,{1\over 2}]$ is such that  $\LK^{-a} \FH$ is well defined. 
But one should keep in mind that all the derived results also hold if we replace $\|\LK^{-a}(\EESRA - \FH)\|_{\rho}$ with $\|\TK^{{1\over 2} - a}(\ESRA - \omega_{\HK})\|_{\HK}$ in the attainable case, i.e., $\FH \in \HR.$
We will report these results in a longer version of this paper.
Convergence with respect to different norms has its strong backgrounds in convex optimization, inverse problems, and statistical learning theory.  Particularly, convergence with respect to target function values and $\HK$-norm has been studied in convex optimization. Interestingly, convergence in $\HK$-norm can imply convergence in target function values (although the derived rate is not optimal), while the opposite is not true.  
\section{Convergence Results}\label{sec:conve}
In this section, we first introduce some basic assumptions and then present convergence results  for projected-regularized algorithms. Finally, we give results for sketched/Nystr\"{o}m regularized algorithms.

\subsection{Assumptions}
In this subsection, we introduce three standard assumptions made in statistical learning theory \cite{steinwart2008support,cucker2007learning,lin2018optimal}. The first assumption relates to a moment condition on the output value $y$.
\begin{as}\label{as:noiseExp}
	There exist positive constants $Q$ and $M$ such that for all $l \geq 2$ with $l \in \mN,$
	\be\label{noiseExp}
	\int_{\mR} |y| ^{l} d\rho(y|x) \leq {1 \over 2} l! M^{l-2} Q^2, 
	\ee
	$\rho_{ X}$-almost surely. 
\end{as}
Typically, the above assumption is satisfied if 
$y$ is bounded almost surely, or if $y = \la\omega_*,x\ra_{\HK} + \epsilon$, where  $\epsilon$ is a Gaussian random variable with zero mean and it is independent from $x$. Condition \eref{noiseExp} implies that the regression function is bounded almost surely, using the Cauchy-Schwarz inequality.

The next assumption relates to the regularity/smoothness of the target function $\FH.$
\begin{as}\label{as:regularity}
	$\FH$ satisfies  \be\label{eq:FHFR}
	\int_{\HK} (\FH(x) - \FR(x))^2 x \otimes x d \rho_X(x) \preceq B^2 \TK,
	\ee 
	and the following H\"{o}lder source condition
	\be\label{eq:socCon}
	\FH = \LK^{\zeta} g_0, \quad \mbox{with}\quad  \|g_0\|_{\rho} \leq R.
	\ee
	Here, $B,R,\zeta$ are non-negative numbers.
\end{as}

Condition \eqref{eq:FHFR} is trivially satisfied if $\FH - \FR$ is bounded almost surely. Moreover, when making a consistency assumption, i.e., $\inf_{\HR} \mcE = \mcE(\FR)$, as that in \cite{smale2007learning,caponnetto2006,caponnetto2007optimal,steinwart2009optimal}, for kernel-based non-parametric regression, it is satisfied with $B=0.$ Condition \eqref{eq:socCon} characterizes the regularity of the target function $\FH$ \cite{smale2007learning}. A bigger $\zeta$ corresponds to a higher regularity and a stronger assumption, and it can lead to a faster convergence rate. Particularly, when $\zeta\geq 1/2$, 
$\FH \in \HR$ \cite{steinwart2008support}. This means that the expected risk minimization \eref{expectedRiskA} has at least one solution  in $\HK$, which is referred to  as the attainable case.
\\
Finally, the last assumption relates to the capacity of the space $\HK$ ($\HR$).
\begin{as}\label{as:eigenvalues}
	For some $\gamma \in [0,1]$ and $c_{\gamma}>0$, $\TK$ satisfies
	\be\label{eigenvalue_decay}
	\tr(\TK(\TK+\lambda I)^{-1})\leq c_{\gamma} \lambda^{-\gamma}, \quad \mbox{for all } \lambda>0.
	\ee
\end{as}

The left hand-side of \eref{eigenvalue_decay} is called degrees of freedom \cite{zhang2006learning}, or effective dimension \cite{caponnetto2007optimal}.
Assumption \ref{as:eigenvalues} is always true for $\gamma=1$ and $c_{\gamma} =\kappa^2$, since
$\TK$ is a trace class operator.
This is referred to as the capacity independent setting.
Assumption \ref{as:eigenvalues} with $\gamma \in[0,1]$ allows to derive better  rates. It is satisfied, e.g.,
if the eigenvalues of $\TK$ satisfy a polynomial decaying condition $\sigma_i \sim i^{-1/\gamma}$, or with $\gamma=0$ if $\TK$ is finite rank.
\subsection{Results for Projected-regularized Algorithms}
We are now ready to state our first result as follows. Throughout this paper,  $C$  denotes a positive constant that depends only on $\kappa^2,c_{\gamma},\gamma,\zeta$ $B,M,Q,R,\tau$ and $\|\TK\|$, and it could be different at its each appearance. Moreover, we write $a_1 \lesssim a_2$ to mean
$a_1\leq C a_2$.

\begin{thm}\label{thm:projerr}
	Under Assumptions \ref{as:noiseExp}, \ref{as:regularity} and \ref{as:eigenvalues},
	let $\lambda = n^{\theta-1}$ for some $\theta\in[0,1]$, $\tau \geq \zeta$, and  $a\in [0,{1\over 2}\wedge \zeta]$.
	Then the following holds with probability at least $1-\delta$ ($0<\delta<1$). \\	
	1) If $\zeta \in [0, 1]$, 
	\begin{align}
	&\|\LK^{-a} ( \EESRA  - \FH) \|_{\rho}
	\lesssim   \lambda^{-a} \log^{2} {3 \over \delta} t_{\theta,n}^{1-a} \nonumber\\
	& \times  \  \left( \lambda^{\zeta}  +
	{1 \over \sqrt{n \lambda^{\gamma}} } +   \lambda^{\zeta-1} \big(\DZI + \DZI^{1-a}\lambda^{a}\big) \right). \label{eq:mnErrPI}
	\end{align}
	2) If $\zeta \geq 1$ and $\lambda\geq n^{-1/2},$
	\begin{align}
	&\|\LK^{-a}(\EESRA - \FH)\|_{\rho} 
	\lesssim  \lambda^{-a} \log^{2} {3 \over \delta} \nonumber \\ & \times   \left( 
	\lambda^{\zeta} + 
	{1 \over \sqrt{n \lambda^{\gamma}} }+ (\DZI + \lambda \DZI^{(\zeta-1)\wedge 1}+ \DZI^{1-a} \lambda^{a})  \right). \label{eq:mnErrPII} 
	\end{align}
	Here, $\DZI$ is the projection error $\|(I - \proj) \TK^{1\over 2}\|^2$ and 
	\be\label{eq:tthetan}
	t_{\theta,n} = [1 \vee (\theta^{-1}\wedge \log n^{\gamma})].
	\ee
\end{thm}
The above result provides high-probability error bounds with respect to variants of norms for projected-regularized algorithms. The upper bound consists of three terms. The first term depends on the regularity parameter $\zeta$, and it arises from estimating bias. The second term depends on the sample size, and it arises from estimating variance. The third term depends on the projection error. Note that there is a trade-off between the bias and variance terms.
Ignoring the projection error, solving this trade-off leads to the best choice on $\lambda$ and the following results.

\begin{corollary}\label{cor:1}
	Under the assumptions and notations of Theorem \ref{thm:projerr}, let $\lambda = n^{-{1 \over 1 \vee (2\zeta+\gamma)}}.$ Then the following holds with probability at least $1-\delta$. \\
	1) If $2\zeta+\gamma \leq 1,$ 
	\begin{align}
	& \|\LK^{-a} (\EESRA  - \FH) \|_{\rho}  \nonumber\\
	& \ \lesssim    n^{-(\zeta-a)} \left(1 + (\gamma \log n)^{1-a})(1 +   \lambda^{-1} \DZI \right) \log^{2} {3 \over \delta} .  \label{eq:terrP1}
	\end{align}	
	2) If $\zeta \in [0, 1]$ and $2\zeta+\gamma >1$, 
	\begin{align}
	\|\LK^{-a} (\EESRA  - \FH) \|_{\rho}  
	\lesssim    n^{-{\zeta-a\over 2\zeta+\gamma}} \left(1 +   \lambda^{-1}\DZI  \right) \log^{2} {3 \over \delta} . \label{eq:terrP2}
	\end{align}
	3) If $\zeta \geq 1,$
	\begin{align}
	& \|\LK^{-a}(\EESRA - \FH)\|_{\rho}  
	\lesssim  \lambda^{-a}  \log^{2} {3 \over \delta} \nonumber \\
	&  \times   \left( n^{-{\zeta \over 2\zeta+\gamma}} + \DZI  \left(1 + \Big( {\lambda \over \DZI} \Big)\DZI^{(\zeta-1)\wedge 1}+ \Big( {\lambda \over \DZI} \Big)^{a} \right)  \right). \label{eq:terrP3} 
	\end{align}
\end{corollary}
Comparing the derived upper bound for projected-regularized algorithms with that for classic regularized algorithms in \cite{lin2018optimal}, we see that the former has an extra term, which is caused by projection. The above result asserts that projected-regularized algorithms perform similarly as classic regularized algorithms if the projection operator is well chosen such that the projection error is small enough.

In the special case that $\st = I$, we get the follow result.
\begin{corollary}\label{cor:2}
	Under the assumptions and notations of Theorem \ref{thm:projerr}, let $\lambda = n^{-{1 \over 1 \vee (2\zeta+\gamma)}}$ and $\st = I$.  Then with probability at least $1-\delta$, 
	\be\label{eq:err}
	\begin{split}
			&\|\LK^{-a} (\EESRA  - \FH) \|_{\rho}\\
			& \ \lesssim  \log^{2} {3 \over \delta}		 \begin{cases}
				n^{-(\zeta-a)} \left(1 + (\gamma \log n)^{1-a} \right), & \mbox{if } 2\zeta+\gamma \leq 1, \\
				n^{-{\zeta-a\over 2\zeta+\gamma}}, & \mbox{if } 2\zeta+\gamma > 1.
			\end{cases}
		\end{split}
	\ee
\end{corollary}
The above result recovers the result derived in \cite{lin2018optimal}. The  convergence rates are optimal as they match the mini-max rates with $\zeta \geq 1/2$ derived in \cite{caponnetto2007optimal,blanchard2016optimal}.

\subsection{Results for Sketched-regularized Algorithms}

 In this subsection, we state results for sketched-regularized algorithms. 

In sketched-regularized algorithms, the range of the projection operator $\proj$ is the subspace  $\overline{range\{\SX^* \skt^*\}},$ where  $\skt \in \mR^{m \times n}$ is a sketch matrix satisfying the following concentration inequality:
For any finite subset $E$ in  $\mR^n$ and for any $t>0,$ 
\be \label{eq:isoPro}
\mP (|\|\skt\ba \|_2^2 - \|\ba\|_2^2 \geq t \|\ba\|_2^2) \leq 2  |E|\mathrm{e}^{-t^2m \over  c_0'\log^{\beta} n}.
\ee
Here, $c_0'$ and $\beta$ are universal non-negative  constants.
 Many matrices satisfy the concentration property.  
 \begin{itemize}
	\item {\bf Subgaussian sketches.}
	Matrices with i.i.d. subgaussian (such as Gaussian or Bernoulli) entries satisfy \eqref{eq:isoPro} with some universal constant $c_0'$ and $\beta = 0$.   More general, if the rows of $\skt$ are
	independent (scaled) copies of an isotropic $\psi_2$ vector, then $\skt$ also satisfies \eqref{eq:isoPro} \cite{mendelson2008uniform}. 
	\item {\bf Randomized orthogonal system (ROS) sketches.}  As noted in \cite{krahmer2011new}, 
	matrix that satisfies restricted isometric property  from compressed sensing with randomized column signs satisfies \eqref{eq:isoPro}. Particularly,  random partial Fourier matrix, or random partial Hadamard matrix with randomized column signs satisfies \eqref{eq:isoPro} with $\beta = 4$ for some universal constant $c_0'$. Using OS sketches has an advantage in computation, as that for suitably chosen orthonormal
	matrices such as the DFT and Hadamard matrices, a matrix-vector product can be executed in $O(n \log m)$ time, in contrast to $O(nm)$ time required for the same operation with generic dense sketches. 
\end{itemize}

The following corollary shows that sketched-regularized algorithms have optimal rates  provided the sketch dimension $m$ is not too small.

\begin{corollary}\label{cor:3}
	Under the assumptions of Theorem \ref{thm:projerr}, let 
	$\HKS =  \overline{range\{\SX^* \skt^*\}},$ where  $\skt \in \mR^{m \times n}$ is a randomized matrix satisfying \eqref{eq:isoPro}. Let
	$\lambda = n^{-{1 \over 1\vee (2\zeta+\gamma)}}$ and 
	\be\label{eq:mnum}
	m \gtrsim  \log^{\beta}n \begin{cases}
		n^{\gamma}\log{3n^{\gamma} \over \delta} \log^2{3\over \delta} & \mbox{if } 2\zeta+\gamma \leq 1,\\
		n^{\gamma(\zeta-a) \over (1-a)(2\zeta+\gamma)} \log^3{3\over \delta}  & \mbox{if } \zeta \geq 1, \\
		n^{\gamma \over 2\zeta+\gamma}  \log^3{3\over \delta} & \mbox{otherwise}.
	\end{cases}
	\ee
	Then with confidence at least $1-\delta,$  the following holds
		\be\label{eq:errP}
	\begin{split}
		&\|\LK^{-a} (\EESRA  - \FH) \|_{\rho}\\
		& \lesssim  \log^{3} {3 \over \delta}	 
			 \begin{cases}
			n^{-(\zeta-a)} \left(1 + (\gamma \log n)^{2-a} \right), & \mbox{if } 2\zeta+\gamma \leq 1, \\
			n^{-{\zeta-a\over 2\zeta+\gamma}}, & \mbox{if } 2\zeta+\gamma > 1.
		\end{cases}
	\end{split}
	\ee
	
\end{corollary}
The above results assert that
sketched-regularized algorithms converge optimally, provided the sketch dimension is not too small, or in another words the error caused by projection is negligible when the sketch dimension is large enough. Note
that the minimal sketch dimension from the above is proportional to the effective dimension $\lambda^{-\gamma}$ up to a logarithmic factor for the case $\zeta \leq 1.$  
\begin{rem}
 Considering only the case $\zeta=1/2$ and $a=0$, \cite{yang2015randomized} provides optimal error bounds for sketched ridge regression within the fixed design setting.  \\
\end{rem}

\subsection{Results for Nystr\"{o}m Regularized Algorithms}
As a byproduct of the paper, using Corollary \ref{cor:1}, we derive the following results for Nystr\"{o}m regularized algorithms.

\begin{corollary}\label{cor:nyReg}
	Under the assumptions of Theorem \ref{thm:projerr}, let $\HKS = \overline{span\{x_1,\cdots,x_m\}}$, $2\zeta+\gamma>1$, and $\lambda = {n^{-{1\over 2\zeta+\gamma}}}$.
	Then  with probability at least $1-\delta$,
	$$\|\LK^{-a} (\EESRA - \FH)\|_{\rho} \lesssim n^{-{\zeta-a \over 2\zeta+\gamma}} \log^3{3\over \delta},$$
	provided that
	\[
	m \gtrsim   (1+ \log n^{\gamma})\begin{cases}
	n^{\zeta-a \over (1-a)(2\zeta+\gamma)}   & \mbox{if } \zeta \geq 1, \\
	n^{1 \over 2\zeta+\gamma}  & \mbox{if }  \zeta \leq 1.
	\end{cases}
	\]
\end{corollary}

\begin{rem}
	1) Considering only the case $1/2 \leq \zeta\leq 1$ and $a=0$, \cite{rudi2015less} provides optimal generalization error bounds for Nystr\"{o}m ridge regression. This result was further extended in \cite{myleiko2017regularized} to a general Nystr\"{o}m regularized algorithm with a general source assumption indexed with an operator monotone function (but only in the attainable cases). Note that as in classic ridge regression, Nystr\"{o}m ridge regression saturates over $\zeta\geq1,$ i.e., it does not have a better rate even for a bigger $\zeta\geq 1$.
	\\
	2) For the case
	$\zeta\geq 1$ and $a=0$, \cite{myleiko2017regularized} provides certain  generalization error bounds for plain Nystr\"{o}m regularized algorithms, but the rates are capacity-independent, and the minimal
	projection dimension $O(n^{2\zeta-1 \over 2\zeta+1})$  is larger than ours (considering the case $\gamma=1$ for the sake of fairness).
\end{rem}

In the above lemma, we consider the plain Nystr\"{o}m subsampling. Using the ALS Nystr\"{o}m subsampling \cite{drineas2012fast,gittens2013revisiting,el2014fast}, we can improve the projection dimension condition to \eqref{eq:mnum}.
\paragraph{ALS Nystr\"{o}m Subsampling}
Let $\bK = \SX\SX^*.$
For $\lambda>0,$ the leveraging scores of $\bK(\bK + \lambda I)$ is the set $\{l_i(\lambda)\}_{i=1}^n$ with
$$l_i (\lambda) =  \left( \bK(\bK + \lambda I)^{-1}\right)_{ii}, \quad  \forall i \in [n].$$
The $L$-approximated leveraging scores (ALS) of $\bK(\bK + \lambda I)$ is a set $\{\hat{l}_i(\lambda)\}_{i=1}^n$ satisfying
$
L^{-1} {l}_{i} (\lambda) \leq \hat{l}_{i}(\lambda) \leq L {l}_{i} (\lambda),
$ 
for some $L \geq 1$.
In ALS Nystr\"{o}m subsampling regime, $\HKS =  \overline{range\{\tilde{x}_1,\cdots, \tilde{x}_m\}},$ where each $\tilde{x}_j$  is i.i.d. drawn according to
$$
\mP \left(\tilde{x} = x_i \right) \sim \hat{l}_{i}(\lambda). 
$$

\begin{corollary}\label{cor:ALSnyReg}
	Under the assumptions of Theorem \ref{thm:projerr}, let $\lambda = {n^{-{1\over (2\zeta+\gamma) \vee 1}}}$ and $\HKS = \overline{range\{\tilde{x}_1,\cdots, \tilde{x}_m\}}$ with $\tilde{x}_j$ drawn following an $L$-ALS Nystr\"{o}m subsampling scheme.
	Then  with probability at least $1-\delta$,
	\eqref{eq:errP} holds 
	provided that
		\be\label{eq:mnumALS}
	m \gtrsim  L^2 \log^2{3\over \delta}  \log{3n^{\gamma} \over \delta} \begin{cases}
		n^{\gamma} [1 \vee \log n^{\gamma}]& \mbox{if } 2\zeta+\gamma \leq 1,\\
		n^{\gamma(\zeta-a) \over (1-a)(2\zeta+\gamma)}  & \mbox{if } \zeta \geq 1, \\
		n^{\gamma \over 2\zeta+\gamma}  & \mbox{otherwise}.
	\end{cases}
	\ee
\end{corollary}

All the results stated in this section will be proved in the next section.

\section{Proof}\label{sec:proof}
In this section, we prove the results stated in  Section \ref{sec:conve}.  We first give some 
deterministic estimates and an analytics result. 
We then give some probabilistic estimates. Applying the probabilistic estimates into the analytics result, we prove the results for projected-regularized algorithms. We finally estimate the projection errors and  present the proof  for sketched-regularized algorithms.
  
\subsection{Deterministic Estimates} 
In this subsection, we introduce some deterministic estimates.  
For notational simplicity, throughout this paper, we denote
$$
\TKL = \TK + \lambda I, \quad \TXL = \TX +  \lambda I.
$$
We define a deterministic vector $\EPSRA$ as follows, 
\be\label{eq:popFunc}
\EPSRA = \GL(\TK) \IK^* \FH.
\ee
The vector $\EPSRA$ is often called population function. 
We introduce the following lemma. The proof is essentially the same as that for Lemma 26 from \cite{lin2018optimalconve}. We thus omit it.
\begin{lemma}\label{lem:popFun}
	Under Assumption \ref{as:regularity}, the following holds. \\
	1) For any $\zeta - \tau\leq a \leq \zeta,$ 	\be\label{eq:trueBias}
	\|\LK^{-a}(\IK \EPSRA - \FH)\|_{\rho} \leq R  \lambda^{\zeta-a} .
	\ee
	2) 
	\be\label{eq:popSeqNorm}
	\|\TK^{a-1/2}\EPSRA\|_{\HK} \leq \tau R\cdot  \begin{cases}
		\RegPar^{\zeta+a -1}, & \text{if } -\zeta \leq a\leq 1-\zeta, \\
		\kappa^{2(\zeta +a-1)},& \text{if } \ a\geq 1 -\zeta.
	\end{cases}
	\ee
\end{lemma}
The above lemma provides some basic properties for the population function. It will be useful for the proof of our main results. The left hand-side of 
\eqref{eq:trueBias} is often called true bias.
\\
Using the above lemma and some basic operator inequalities, we can prove the following analytic, deterministic result.
\begin{pro}\label{pro:anaRes}
	Under Assumption \ref{as:regularity},
	let $$
	1\vee	\|\TKL^{1\over 2} \TXL^{-{1\over 2}}\|^2 \vee \|\TKL^{-{1\over 2}} \TXL^{{1\over 2}}\|^2 \leq \DZF,
	$$
	$$
	\quad \|\TKLL^{-1/2} [(\TX \EPSRA - \SX^* \NOutputs) - (\TK \EPSRA - \IK^* \FH)] \|_{\HK} \leq \DZS ,
	$$
	$$
	\|\TK - \TX\|	\leq \DZT, 
	$$
	$$
	\|\TKL^{-{1\over 2}} (\TK - \TX)\| \leq \DZN,
	$$
	$$
	\|(I-\proj) \TK^{1\over 2}\|^2 = \DZI.
	$$
	Then, for any $0 \leq a\leq \zeta \wedge {1\over 2},$ the following holds.\\
	1) If $\zeta \in [0, 1]$, 
	\begin{align}
	& \|\LK^{-a} (\IK \ESRA  - \FH) \|_{\rho}
	\leq   \tau\lambda^{-a} \DZF^{1-a}\nonumber\\ 
	& \ \ \times \Big( \DZS +   2(\tau+1)R \lambda^{\zeta}+  \tau R \lambda^{\zeta-1}   (\DZI + \DZI^{1-a}\lambda^{a})  \Big). \label{eq:totErrBI1}
	\end{align}
	2) If $\zeta \geq 1,$
	\begin{align}
	&\|\LK^{-a}(\IK\ESRA - \FH)\|_{\rho} 
	\leq   \tau \lambda^{-a} \DZF^{1-a}\nonumber \\
	& \ \ \times \left( \DZS +   3R \lambda^{\zeta}+ \kappa^{2(\zeta-1)} R \big(\kappa \tau  \DZN + \tau \DZI  \right.\nonumber \\ 
	& \ \ \ \left. + \lambda (\DZT + \DZI)^{(\zeta-1)\wedge 1} + \lambda^{{1\over 2}}  \DZT^{(\zeta-{1\over 2})\wedge 1} + \DZI^{1-a} \lambda^{a}\big)  \right). \label{eq:totErrBII1}
	\end{align}
\end{pro}
The above proposition is key to our proof. 
The proof of the above proposition for the case $\zeta\leq 1$ borrows ideas from  \cite{smale2007learning,caponnetto2007optimal,rudi2015less,myleiko2017regularized,lin2018optimal}, whereas the key step is an error decomposition from \cite{lin2018optimalconve}.
 Our novelty lies in the proof for the case $\zeta\geq1$, see the appendix for further details.

\subsection{Proof for Projected-regularized Algorithms}
To derive total error bounds from Proposition \ref{pro:anaRes}, it is necessary to develop probabilistic estimates for the random quantities
$\DZF$, $\DZS$, $\DZT$ and $\DZN$. We thus introduce the following four lemmas.
\begin{lemma}[\cite{lin2018optimalconve}] \label{lem:operDifRes}
	Under Assumption \ref{as:eigenvalues}, 
	let $\delta\in(0,1)$, $\lambda= n^{-\theta}$ for some $\theta\geq 0$, and
	\be\label{eq:aa}\begin{split}
	&a_{n,\delta,\gamma}(\theta) \\
	& =  8\kappa^2 \left(\log{ {4\kappa^2(c_{\gamma}+1) }\over \delta \|\TK\|} + \theta \gamma \min\left({1 \over \mathrm{e}(1-\theta)_+},\log n\right)\right).
	\end{split}
	\ee
	We have with probability at least $1-{\delta},$
	\bea
	\| (\TK+\lambda)^{1/2}(\TX+\lambda)^{-1/2}\|^2 \leq 3 a_{n,\delta,\gamma}(\theta) (1 \vee n^{\theta-1}),
	\eea
	and 
	\bea
	\| (\TK+\lambda)^{-1/2}(\TX+\lambda)^{1/2}\|^2 \leq {4\over 3} a_{n,\delta,\gamma}(\theta) (1 \vee n^{\theta-1}).
	\eea
\end{lemma}

\begin{lemma}\label{lem:statEstiOper}
	Let $0<\delta<1/2.$ It holds with probability at least $1-\delta:$
	\bea
	\	\|\TK - \TX\|_{HS} \leq { 6 \kappa^2 \over \sqrt{{n}}} \log {2\over \delta}.
	\eea
	Here, $\|\cdot\|_{HS}$ denotes the Hilbert-Schmidt norm.
\end{lemma}

\begin{lemma}\label{lem:effDifOper}
Under Assumption \ref{as:eigenvalues},	let $0<\delta<1/2.$ It holds with probability at least $1-\delta:$
\bea
\|\TKL^{-{1\over 2}} (\TK - \TX)\|_{HS} \leq 2\kappa \left( {2\kappa \over n\sqrt{\lambda}} + \sqrt{c_{\gamma} \over n \lambda^{\gamma}}
\right)\log{2\over \delta}.
\eea
\end{lemma}

The proof of the above lemmas can be done simply applying concentration inequalities for sums of Hilbert-space-valued random variables. We refer to \cite{lin2017optimal} for the proofs.

\begin{lemma} \cite{lin2018optimal}\label{lem:samAppErr}
	Under Assumptions \ref{as:noiseExp}, \ref{as:regularity} and \ref{as:eigenvalues}, let  $\EPSRA$ be given by \eref{eq:popFunc}. For all $\delta \in ]0,1/2[,$ the following holds with probability at least $1-{\delta}:$
	\be
	\begin{split}
	&\|\TKLL^{-1/2} [(\TX \EPSRA - \SX^* \NOutputs) - (\TK \EPSRA - \IK^* \FR)] \|_{\HK}  \\
	&\leq    \left({C_1 \over n \lambda^{{1\over 2} \vee (1-\zeta)}} + \sqrt{ {C_2 \lambda^{2\zeta}\over n\lambda} + {C_3 \over n \lambda^{\gamma}} }\right) 
	\log{2\over \delta}.
	\end{split}
	\ee
	Here, $C_1 = 4(M + R \kappa^{(2\zeta-1)_+}), C_2 = 96 R^2 \kappa^2$ and $C_3 = 32 (3B^2+ 4Q^2)c_{\gamma}.$
\end{lemma}
With the above probabilistic estimates and the analytics result, Proposition \ref{pro:anaRes}, we are now ready prove  results for projected-regularized algorithms.
\begin{proof}[Proof of Theorem \ref{thm:projerr}]
	We use Proposition \ref{pro:anaRes} to prove the result. We thus need to estimate $\DZF,$ $\DZS$, $\DZT$ and $\DZN$. 
	Following from Lemmas \ref{lem:operDifRes}, \ref{lem:statEstiOper}, \ref{lem:effDifOper} and \ref{lem:samAppErr}, with $n^{-1} \leq \lambda \leq 1,$  we know that with probability at least $1-\delta,$ 
	\be\label{eq: DZFbound}
	\DZF \lesssim t_{\theta,n}\log {3 \over \delta},
	\ee
	$$\DZS \lesssim  \left({1 \over n \lambda^{{1\over 2} \vee (1-\zeta)}} + 
	\lambda^{\zeta} + 
	{1 \over \sqrt{n \lambda^{\gamma}} }\right)
	\log{3\over \delta},
	$$
	\be\label{eq:DZTbound}
	\DZT \lesssim {1\over \sqrt{n}} \log {3\over \delta},
	\ee
	$$
	\DZN \lesssim {1 \over \sqrt{n\lambda^{\gamma}}} \log {3\over \delta}.
	$$
	The results thus follow by introducing the above estimates into \eqref{eq:totErrBI1} or \eqref{eq:totErrBII1}, combining with a direct calculation and $1/n \leq \lambda \leq 1$.
\end{proof}

\subsection{Proof for Sketched-regularized Algorithms}
In order to use Corollary \ref{cor:1}  for sketched-regularized algorithms, we need to estimate the projection error. The basic idea is to approximate the projection error in terms of its `empirical' version, $\|(I -\proj) \TX^{1\over 2}\|^2$.
The estimate for $\|(I -\proj) \TX^{1\over 2}\|^2$
is quite lengthy and it is divided into several steps.
%
\begin{lemma}\label{lem:OperDiffProd}
	Let $0<\delta <1$ and $\theta \in[0,1].$
	Given a fix $\bx \subseteq \HK^n$, assume that for $\lambda >0$,
	\be\label{eq:empEffDim}
	\mcNx : = 	\tr ((\TX+\lambda)^{-1}\TX) \leq b_{\gamma} \lambda^{-\gamma}
	\ee
	holds for some $b_{\gamma}>0$, $\gamma \in[0,1]$.
	Then there exists a subset $U_{\bx} $ of $\mR^{m\times n}$ with measure at least $1-\delta$, such that for all $\skt \in U_{\bx}$, 
	$$
	\|(I - \proj)\TX^{1\over 2}\|^2 \leq  6\lambda,
	$$
	provided that
	\be\label{eq:subsamLev}
	m \geq 144c_0' \log^{\beta}n \lambda^{-\gamma}\log{3 \over \delta} \left( 1 + 12 b_{\gamma}   \right) \log 2 .
	\ee	
\end{lemma}

Under the condition \eqref{eq:empEffDim}, Lemma \ref{lem:OperDiffProd} provides an upper bound for $\|(I - \proj)\TX^{1\over 2}\|$, which will be used to control the projection error using the following lemma.

\begin{lemma}\label{lem:proje2te}
	Let $\proj$ be a projection operator in a Hilbert space $\HK$, and $A$, $B$ be two semidefinite positive  operators on $\HK.$
	For any $0 \leq s,t \leq {1\over 2}$, we have 
	$$
	\|A^{s} (I - \proj) A^{t}\| \leq \|A - B\|^{s+t} + \|B^{1\over 2}(I - \proj)B^{1\over 2}\|^{s+t}.
	$$
\end{lemma}

 The left-hand side of \eqref{eq:empEffDim} is called empirical effective dimension. It can be estimated as follows.

\begin{lemma}\label{lem:empEffdim}
	Under Assumption \ref{as:eigenvalues},
	let $\lambda = n^{-\theta}$ for some $\theta\in[0,1]$ and $0<\delta <1$. With confidence $1-\delta$, 
	\be\label{eq:empEffdim}
	\begin{split}
	&\tr ((\TX+\lambda)^{-1}\TX)  \\
	& \leq 3 (4\kappa^2 + 2\kappa\sqrt{c_{\gamma}} + c_{\gamma}) \log {4 \over \delta} a_{n,\delta/2,\gamma}(\theta)  \lambda^{-\gamma},
	\end{split}
	\ee
	where $a_{n,\delta/2,\gamma}(\theta)$ is given as in Lemma \ref{lem:operDifRes}.
\end{lemma}
The above lemma improves Proposition 1 of \cite{rudi2015less}. It does not require the extra assumption that the sample size is large enough, and our proof is simpler.

Now we are ready to estimate the projection error and give the proof for sketched-regularized algorithms.
\begin{proof}[Proof of Corollary \ref{cor:3}]
	Let $\lambda' = n^{-\theta'},$ with
	\[
	\theta' = \begin{cases}
	1, & \mbox{if } 2\zeta+\gamma \leq 1,\\
	{\zeta-a \over (1-a)(2\zeta+\gamma)}, & \mbox{if } \zeta\geq 1, \\
	{1\over 2\zeta+\gamma}, & \mbox{otherwise}
	\end{cases}
	\]
	Following from Corollary \ref{cor:1}, 
	Lemmas \ref{lem:operDifRes} and \ref{lem:empEffdim}, we know that there exists a subset $\Omega$ of $Z^n$ with measure at least $1-3\delta$, such that for all $\bz \in \Omega$, 
	\eqref{eq:terrP1} (or \eqref{eq:terrP2}, or \eqref{eq:terrP3}), 
	\eqref{eq:empEffdim} (with $\theta$ and $\lambda$ replaced by $\theta'$ and $\lambda'$ in 	\eqref{eq:empEffdim}, respectively), and 
	\be\label{eq:DZFboundB}
	\|\TK_{\bx \eta}^{-1/2} \TK_{\eta}^{1\over 2}\|^2 \lesssim  {\log{3\over \delta}}
	\begin{cases}
	n^{-\theta'} & \mbox{for } \theta'<1, \\
	{(1 \vee \log n^{\gamma}) n^{-1}} &  \mbox{for } \theta' = 1.
	\end{cases} 
	\end{equation}
	Here, $\eta = n^{-\theta'}$ if $\theta' <1$ or $\eta = n^{-1}$ otherwise.
\\		
	For any $\bz \in \Omega$, using Lemma \ref{lem:OperDiffProd} with 
	\[
	\begin{split}
	b_{\gamma} = & 3 (4\kappa^2 + 2\kappa\sqrt{c_{\gamma}} + c_{\gamma}) \log {4 \over \delta} a_{n,\delta/2,\gamma}(\theta') \\ \lesssim &
	(1\vee [(1-\theta')^{-1}\wedge \log n^{\gamma}] + \log{3\over \delta}) \log {3\over \delta},
	\end{split}
	\]
	we know that there exists a subset $U_{\bz}$ of $\mR^{m \times n}$ with measure at least $1-\delta,$ such that for all $\skt \in U_{\bz}$,
	\be\label{eq:t4}
	\|(I - \proj)\TX^{1\over 2}\|^2 \lesssim {1\over n^{\theta'}},
	\ee
	provided
	$
	m \gtrsim n^{\theta \gamma}	\log^{\beta} n  \log{3 \over \delta} b_{\gamma} ,
	$
	which is guaranteed by Condition \eqref{eq:mnum}.
Note that,
	\[
	\begin{split}
	\|(I -\st) \TK^{1\over 2}\|^2 \leq & \|(I -\st) \TK_{\bx\eta}^{1\over 2}\|^2 \|\TK_{\bx \eta}^{-{1\over 2}} \TK_{\eta}^{1\over 2}\|^2 \\
	 \leq &  \left( \|(I -\st) \TK_{\bx}^{1\over 2}\|^2 + \eta \right) \|\TK_{\bx \eta}^{-{1\over 2}} \TK_{\eta}^{1\over 2}\|^2.
	 \end{split}
	\]
	Introducing with \eqref{eq:DZFboundB} and \eqref{eq:t4}, combining with \eqref{eq:terrP1} (or \eqref{eq:terrP2}, or \eqref{eq:terrP3}), 
	and by a simple calculation, one can prove the desired results.
\end{proof}
{\it The proof of Corollaries \ref{cor:nyReg} and \ref{cor:ALSnyReg}} will be given in the appendix due to space limitation.
\section{Conclusion}
  In this paper, we prove optimal statistical results 
  with respect to variants of norms for sketched or Nystr\"{o}m regularized algorithms.
  Our contributions are mainly on theoretical aspects. 
  First,  our results for sketched-regularized algorithms generalize previous results \cite{yang2015randomized} from the fixed design setting to the random design setting.  
  Moreover, our results involve the regularity/smoothness of the target function and thus can have a faster convergence rate. 
  Second, our results cover the non-attainable cases, which have not been studied before for both Nystr\"{o}m and sketched regularized algorithms.
  Third, our results provide the first optimal, capacity-dependent rates even when $\zeta\geq 1.$ This may suggest that sketched/Nystr\"{o}m regularized algorithms have certain advantages in comparison with distributed learning algorithms \cite{zhang2015divide}, as the latter suffer a saturation effect over $\zeta=1$.  A future direction is to
  extend our analysis to learning with random features, see \cite{sriperumbudur2017approximate,lin2018generalization} and references therein.

\section*{Acknowledgements}
 This work was sponsored by the Department of the Navy, Office of Naval Research (ONR) under a grant number N62909-17-1-2111.  It has also received funding from Hasler Foundation Program: Cyber Human Systems (project number 16066), and from the European Research Council (ERC) under the European Union¡¯s Horizon 2020 research and innovation program (grant agreement n 725594-time-data).

\bibliography{projreg}
\bibliographystyle{icml2018}

\newpage
\onecolumn
\appendix
 
 \section*{Supplementary:  Optimal Rates of Sketched-regularized Algorithms for Least-squares Regression over Hilbert Spaces}
In this appendix, we first prove the lemmas stated in  Section \ref{sec:proof} and Corollaries \ref{cor:nyReg} and \ref{cor:ALSnyReg}. We then review how the regression setting considered in this paper covers non-parametric regression with kernel methods.

\section{Proofs for Section \ref{sec:proof}}

For notational simplicity, we denote
\be\label{eq:residual}\RL(u) = 1 - \GL(u)u,\ee
and 
$$
\mcN(\lambda) = \tr(\TK (\TK + \lambda)^{-1}).
$$   
To proceed the proof, we need some basic operator inequalities. 

\begin{lemma} \cite{fujii1993norm}
	\label{lem:operProd}
	Let $A$ and $B$ be two positive bounded linear operators on a separable Hilbert space. Then
	\bea
	\|A^s B^s\| \leq \|AB\|^s, \quad\mbox{when } 0\leq s\leq 1.
	\eea
\end{lemma}

\begin{lemma}\label{lem:sss}
	Let $H_1,H_2$ be two separable Hilbert spaces and  $\mcS: H_1 \to H_2$ a compact operator. Then for any function $f:[0,\|\mcS\|] \to [0,\infty[$,
	$$
	f(\mcS\mcS^*)\mcS = \mcS f(\mcS^*\mcS).
	$$
\end{lemma}
\begin{proof}
	The result can be proved using singular value decomposition of a compact operator.
\end{proof}

\begin{lemma}\label{lem:operDiff}
	Let $A$ and $B$ be two non-negative bounded linear operators on a separable Hilbert space with $\max(\|A\|,\|B\|) \leq \kappa^2$ for some non-negative $\kappa^2.$
	Then for any $\zeta>0,$
	\be
	\|A^\zeta - B^\zeta\| \leq C_{\zeta,\kappa} \|A -B\|^{\zeta \wedge 1},
	\ee
	where
	\be
	C_{\zeta,\kappa} = \begin{cases}
		1   & \mbox{when } \zeta \leq 1,\\
	2	\zeta \kappa^{2\zeta - 2} &\mbox{when } \zeta >1.
	\end{cases}
	\ee
\end{lemma}
\begin{proof}
	The proof is based on the fact that $u^{\zeta}$ is operator monotone if $0<\zeta\leq 1$. While for
	$\zeta \geq 1$,  the proof can be found in, e.g., \cite{dicker2016kernel}.
\end{proof}

\begin{lemma}\label{lem:hansen_ineq}
	Let $X$ and $A$ be bounded linear operators on a separable Hilbert space. Suppose that
	$X \succeq  0$ and $\|A\| \leq 1$. Then for any $s\in[0,1],$
	$$
	X^* A^s X \leq (X^* A X)^s.
	$$
\end{lemma}
\begin{proof}
	Following from  \cite{hansen1980operator} and the fact that the function $u^s$ with $s\in[0,1]$ is operator monotone.
\end{proof}

\subsection{Proof of Proposition \ref{pro:anaRes}}
	Adding and subtracting with the same term, and using the triangle inequality, we have
	\[
	\begin{split}
	\|\LK^{-a} (\IK \ESRA  - \FH) \|_{\rho} 
	\leq    \| \LK^{-a} \IK (\ESRA -  \EPSRA) \|_{\rho} + \| \LK^{-a}   (\IK \EPSRA - \FH) \|_{\rho}. 
	\end{split}
	\]
	Applying Part 1) of Lemma \ref{lem:popFun} to bound the last term, with $0\leq a \leq \zeta,$ 
	\begin{align*}
	\|\LK^{-a} (\IK \ESRA  - \FH) \|_{\rho} 
	\leq&   \| \LK^{-a} \IK(\ESRA -  \EPSRA)\|_{\rho} + R\lambda^{\zeta-a} \\ 
	\leq&  \| \LK^{-a} \IK \TK^{a-{1\over 2}}\| \|\TK^{{1\over 2}-a}(\ESRA -  \EPSRA) \|_{\HK} + R\lambda^{\zeta-a}  . 
	\end{align*}
	Using the spectral theorem for compact operators, $\LK = \IK \IK^*$, and $\TK = \IK^* \IK,$
	we have 
	$$\|\LK^{-a} \IK \TK^{a- {1\over 2}} \| \leq 1,$$
	and thus 
	$$
	\|\LK^{-a} (\IK \ESRA  - \FH) \|_{\rho} 
	\leq  \|\TK^{{1\over 2}-a}(\ESRA -  \EPSRA) \|_{\HK} + R\lambda^{\zeta-a} . 
	$$
	Adding and subtracting with the same term, and using the triangle inequality,
	\begin{align*}
	\|\LK^{-a} (\IK \ESRA  - \FH) \|_{\rho} 
	\leq   \| \TK^{{1\over 2}-a} (\ESRA - \proj \EPSRA) \|_{\HK} + \| \TK^{ {1\over 2}-a}(I - \proj) \EPSRA \|_{\HK} + R\lambda^{\zeta-a} . 
	\end{align*}
	Since $\st$ is an orthogonal projected operator and $a \in[0, {1\over 2}]$, we have
	\begin{align*}
	&\| \TK^{ {1\over 2}-a}(I - \proj) \EPSRA \|_{\HK} \\
	=  &\| \TK^{ {1\over 2}(1-2a)}(I - \proj)^{1-2a} (I-P) \EPSRA \|_{\HK} \nonumber\\ 
	\leq &\| \TK^{ {1\over 2}(1-2a)}(I - \proj)^{1-2a}\| \|(I-P)\TK^{1\over 2}\| \|\TK^{-{1\over 2}} \EPSRA \|_{\HK}  \nonumber \\
	\leq &\| \TK^{ {1\over 2}}(I - \proj)\|^{1-2a} 
	\|(I - \proj) \TK^{ {1\over 2}}\| \tau R \kappa^{2(\zeta-1)_+} \lambda^{(\zeta-1)_-} \\
	= &\DZI^{1-a} \tau R \kappa^{2(\zeta-1)_+} \lambda^{(\zeta-1)_-} , 
	\end{align*}
	(where  for the last second inequality, we used Lemma \ref{lem:operProd} and Part 2) of Lemma \ref{lem:popFun}), and we subsequently get that 
	\begin{align*}
	\|\LK^{-a} (\IK \ESRA  - \FH) \|_{\rho} 
	\leq   \| \TK^{{1\over 2}-a} (\ESRA - \proj \EPSRA) \|_{\HK}  
 + \tau R  \kappa^{2(\zeta-1)_+} \lambda^{(\zeta-1)_-}  \DZI^{1-a}  + R\lambda^{\zeta-a} . 
	\end{align*}
	Since for all $\omega \in \HK,$ and $a \in [0,{1\over 2}],$
	\begin{align*}
	\|\TK^{{1\over 2}-a}\omega \|_{\HK} \leq& \|\TKL^{{1\over 2}-a} \TXL^{a-{1\over 2}} \| \|\TXL^{{1\over 2}-a}\omega \|_{\HK}  \\
	\leq& \lambda^{-a}  \|\TKL^{{1\over 2}-a} \TXL^{a-{1\over 2}} \| \|\TXL^{{1\over 2}}\omega \|_{\HK} \\
	\leq& \lambda^{-a}  \|\TKL^{1\over 2} \TXL^{-{1\over 2}} \|^{1-2a} \|\TXL^{{1\over 2}}\omega \|_{\HK} \\
	\leq & \lambda^{-a} \DZF^{{1\over 2}-a} \|\TXL^{{1\over 2}}\omega \|_{\HK}
	\end{align*}
	(where we used
	 Lemma \ref{lem:operProd} for the last second inequality),
	we get 
	\begin{align}
\|\LK^{-a} (\IK \ESRA  - \FH) \|_{\rho}
	\leq  \lambda^{-a} \DZF^{{1\over 2}-a}  \|\TXL^{{1\over 2}}(\ESRA -  \st\EPSRA)\|_{\HK}
	 + \tau R   \kappa^{2(\zeta-1)_+} \lambda^{(\zeta-1)_-}  \DZI^{1-a}  + R\lambda^{\zeta-a} .\label{eq:totErr1}
	\end{align}
	In what follows, we estimate $\|\TXL^{{1\over 2}}(\ESRA -  \st\EPSRA)\|_{\HK}$.
	
	Introducing with \eqref{eq:ESRA},  with $\st^2 =\st,$
	\[
	\begin{split}
	\|\TXL^{{1\over 2}}(\ESRA -  \st\EPSRA)\|_{\HK}  = \|\TXL^{{1\over 2}}\st (\GL(\st \TX \st) \st \SX^* \bby  -  \st \EPSRA)\|_{\HK}.
	\end{split}
	\]
	Since for any $\omega\in\HK$, 
	\begin{align*}
	\|\TXL^{{1\over 2}}\st \omega\|_{\HK}^2 = \la \st \TXL\st \omega, \omega  \ra_{\HK} 
	\leq  \la (\st \TX \st + \lambda) \omega, \omega  \ra_{\HK} = \|(\st \TX \st + \lambda)^{{1\over 2}}\omega\|_{\HK}^2,
	\end{align*}
	and we thus get 
	$$
	\|\TXL^{{1\over 2}}(\ESRA -  \st\EPSRA)\|_{\HK} \leq	\|\aol^{{1\over 2}}(\GL(\ao) \st \SX^* \bby  -  \st\EPSRA)\|_{\HK},
	$$
	where we denote
	\be\label{eq:ao}
	\ao = \st \TX \st,\quad \aol = \ao + \lambda.
	\ee
	Subtracting and adding with the same term, and applying the triangle inequality, with the notation $\RL$ given by  \eqref{eq:residual} and $\st^2 = \st$,  we have
	\begin{align}
	\|\TXL^{{1\over 2}}(\ESRA -  \st\EPSRA)\|_{\HK} \leq \|\underbrace{\aol^{1\over 2}\GL(\ao)\st (\SX^* \NOutputs - \TX \st \EPSRA )}_{\bf Term.A}\|_{\HK} + \| \underbrace{ \aol^{1\over 2} \RL(\ao) \st\EPSRA}_{\bf  Term.B} \|_{\HK}. \label{eq:1ErrTM}
	\end{align}
	We will estimate the above two terms of the right-hand side.\\
	{\bf Estimating $\|\TA\|_{\HK}$:} \\
	Note that	
	\begin{align*}
	&(\aol^{{1\over 2}} \GL(\ao) \st \TXL^{1\over 2}) (\aol^{{1\over 2}} \GL(\ao) \st \TXL^{1\over 2})^* \\& = \aol^{{1\over 2}} \GL(\ao) (\ao +  \lambda \st^2) \GL(\ao) \aol^{{1\over 2}} \\
	& \preceq \big[\aol \GL(\ao)\big]^2, 
	\end{align*}
	where we used $\st^2 = \st \preceq I$ for the last inequality. Thus, combing with $\|A\| = \|A^*A\|^{1\over 2}$,  
	$$
	\|	\aol^{1\over 2} \GL(\ao) \st \TXL^{1\over 2}\| \leq \| \aol \GL(\ao)\|.
	$$
	Using the spectral theorem, with $\|\ao\| \leq \|\TX\| \leq  \kappa^2 $ (implied by \eqref{eq:TXbound}), and then applying \eqref{eq:GLproper1}, 
	$$
	\|	\aol^{1\over 2} \GL(\ao) \st \TXL^{1\over 2}\| \leq \sup_{u \in [0,\kappa^2]}|(u+\lambda ) \GL(u)| \leq \tau.
	$$
	Using the above inequality, and by a simple calculation,
		\begin{align*}
	\|\TA\|_{\HK} \leq & \|	\aol^{1\over 2} \GL(\ao) \st \TXL^{1\over 2}\| 
	\|\TXL^{-{1\over 2}}(\SX^* \Outputs - \TX \st \EPSRA ) \| \leq  \tau \|\TXL^{-{1\over 2}}(\SX^* \NOutputs - \TX \st \EPSRA ) \|.
	\end{align*}
		Adding and subtracting with the same terms, and using the triangle inequality, 
	\begin{align*}
	 \|\TA\|_{\HK}  
	\leq&  \tau \|\TXL^{-{1\over 2}}(\SX^* \NOutputs - \TX \EPSRA)\|_{\HK}+  \tau \| \TXL^{-{1\over 2}}\TX( I - \st) \EPSRA ) \|_{\HK} \\ 
	\leq&\tau \|\TXL^{-{1\over 2}}\TKL^{1\over 2}\| \|\TKL^{-{1\over 2}}(\SX^* \NOutputs - \TX \EPSRA)\|_{\HK}+ \tau \|  \TXL^{-{1\over 2}}\TX( I - \st) \EPSRA ) \|_{\HK} \\
	\leq&  \tau \DZF^{1\over 2} \|\TKL^{-{1\over 2}}(\SX^* \NOutputs - \TX \EPSRA)\|_{\HK}+  \tau \|  \TXL^{-{1\over 2}}\TX( I - \st) \EPSRA ) \|_{\HK}\\
	\leq& \tau \DZF^{1\over 2} (\DZS +  \|\TKL^{-{1\over 2}}(\TK \EPSRA - \IK^* \FH)\|_{\HK})+  \tau \|  \TXL^{-{1\over 2}}\TX( I - \st) \EPSRA ) \|_{\HK}\\
	\leq&  \tau \DZF^{1\over 2} (\DZS +  \|\TKL^{-{1\over 2}}\IK^* \| \|\IK \EPSRA -  \FH\|_{\rho})+  \tau \|  \TXL^{-{1\over 2}}\TX( I - \st) \EPSRA ) \|_{\HK},
	\end{align*}	
	where we used $\TK = \IK^* \IK$ for the last inequality.
	Applying Part 1) of Lemma \ref{lem:popFun} and $\|\TKL^{-{1\over 2}}\IK^* \| \leq 1$,
	\begin{align}\label{eq:termAInterm}
	\|\TA\|_{\HK}
	\leq  \tau \DZF^{1\over 2} (\DZS +   R \lambda^{\zeta})+  \tau \|  \TXL^{-{1\over 2}}\TX( I - \st) \EPSRA ) \|_{\HK}.
	\end{align}
	In what follows, we estimate $\|  \TXL^{-{1\over 2}}\TX( I - \st) \EPSRA \|_{\HK}$, considering two different cases.\\
	{\it Case $\zeta\leq 1$.}\\
	We have
	\begin{align*}
	\|  \TXL^{-{1\over 2}}\TX( I - \st) \EPSRA \|_{\HK} \leq \| \TXL^{-{1\over 2}} \TX\TXL^{-1\over 2}\| \|\TXL^{1\over 2} \TKL^{-{1\over 2}}\| \|\TKL^{1\over 2} ( I - \st) \EPSRA\|_{\HK} \leq \DZF^{1\over 2} \|\TKL^{1\over 2} ( I - \st) \EPSRA\|_{\HK} .
	\end{align*}
	Since $\st$ is a projection operator, 
	$(I-\st)^2 = I -\st$, and we thus have
	\begin{align*}
	\|  \TXL^{-{1\over 2}}\TX( I - \st) \EPSRA  \|_{\HK} \leq \DZF^{1\over 2} \|\TKL^{1\over 2} ( I - \st)\| \|(I-\st)\TK^{1\over 2}\|  \|\TK^{-{1\over 2}} \EPSRA\|_{\HK} \leq  
\tau	\DZF^{1\over 2} \|\TKL^{1\over 2} ( I - \st)\| \DZI^{1\over 2}  R\lambda^{\zeta-1},
	\end{align*}
	where for the last inequality, we used Part 2) of Lemma \ref{lem:popFun}.
	Note that for any $\omega \in \HK$ with $\|\omega\|_{\HK}=1,$
	$$
	\|\TKL^{1\over 2} ( I - \st)\omega\|_{\HK}^2 = \la \TKL(I - \st) \omega, (I - \st)\omega \ra_{\HK} =  \|\TK^{1\over 2}(I - \st) \omega\|_{\HK}^2 + \lambda \|(I-\st)\omega\|_{\HK}^2 \leq \|\TK^{1\over 2} (I - \st)\|^2 + \lambda \leq \DZI + \lambda.
	$$
It thus follows that 
\be\label{eq:cc}
\|\TKL^{1\over 2} ( I - \st)\|_{\HK}\leq (\DZI + \lambda)^{1\over2},\ee
and thus
\begin{align*}
\|  \TXL^{-{1\over 2}}\TX( I - \st) \EPSRA  \|_{\HK}  \leq  
  \DZF^{1\over 2} (\DZI + \lambda)\tau R\lambda^{\zeta-1}.
\end{align*}
Introducing the above into \eqref{eq:termAInterm}, we know that $\TA$ can be estimated as ($\zeta\leq 1$)
	\begin{align}\label{eq:termAI}
\|\TA\|_{\HK}
\leq  \tau \DZF^{1\over 2} \left( \DZS +   (\tau+1)R \lambda^{\zeta} +  \tau R\lambda^{\zeta-1} \DZI\right).
\end{align}	
	{\it Case $\zeta\geq 1$.}\\
We first have
	\begin{align*}
 \|\TXL^{-{1\over 2}}\TX( I - \st) \EPSRA \|_{\HK}  
\leq&  \DZF^{1\over 2}  \| \TKL^{-{1\over 2}}\TX( I - \st) \EPSRA ) \|_{\HK} \\ 
\leq&  \DZF^{1\over 2} \left(  \| \TKL^{-{1\over 2}}(\TX-\TK)( I - \st) \EPSRA \|_{\HK} + \| \TKL^{-{1\over 2}}\TK ( I - \st) \EPSRA \|_{\HK} \right)\\
\leq&  \DZF^{1\over 2}  \left(  \DZN \|( I - \st) \EPSRA  \|_{\HK} + \| \TK^{1\over 2} ( I - \st) \EPSRA \|_{\HK} \right).
\end{align*}	
Since $\st$ is a projection operator, 
$(I-\st)^2 = I -\st$,  we thus have
\begin{align*}
\|\TXL^{-{1\over 2}}\TX( I - \st) \EPSRA \|_{\HK}
\leq &  \DZF^{1\over 2} \left( \DZN \|I - \st\| \|\TK^{1\over 2}\| \|\TK^{-{1\over 2}} \EPSRA  \|_{\HK} + \| \TK^{1\over 2} ( I - \st)\| \| ( I - \st) \TK^{1\over 2}\| \|\TK^{-{1\over 2}}\EPSRA \|_{\HK}  \right) \\
\leq & 
\DZF^{1\over 2} \left( \kappa\DZN  + \DZI   \right) \|\TK^{-{1\over 2}} \EPSRA  \|_{\HK} ,
\end{align*}
where we used \eqref{eq:TKBound} for the last inequality. 
Applying Part 2) of Lemma \ref{lem:popFun},  we get
\begin{align*}
\|\TXL^{-{1\over 2}}\TX( I - \st) \EPSRA \|_{\HK}
\leq & 
\DZF^{1\over 2} \left( \kappa\DZN  + \DZI   \right) \tau \kappa^{2(\zeta-1)}R.
\end{align*}
Introducing the above into \eqref{eq:termAInterm}, we get for $\zeta\geq 1,$
\begin{align}\label{eq:termAII}
\|\TA\|_{\HK}
\leq & 
\tau \DZF^{1\over 2} \left( \DZS +   R \lambda^{\zeta} +\left( \kappa\DZN  + \DZI   \right) \tau \kappa^{2(\zeta-1)}R \right).
\end{align}
	{\bf Estimating $\|\TB\|_{\HK}$: }\\ 
	We estimate $\|\TB\|_{\HK}$, considering two different cases.\\
	{\it Case I: $\zeta \leq 1.$} \\
	We first have
	\begin{align*}
	\aol^{1\over 2} \RL(\ao) \st \TXL^{1\over 2}(\aol^{1\over 2} \RL(\ao) \st \TXL^{1\over 2})^* 
	= &  \aol^{1\over 2} \RL(\ao) (\ao + \lambda \st^2)\RL(\ao) \aol^{1\over 2}\\ 
	\preceq& \big(\RL(\ao) \aol \big)^2,
	\end{align*}
	where we used $\st^2 = \st \preceq I$ for the last inequality.  Thus, according to $\|A\| = \|AA^*\|^{1\over 2},$
	$$
	\|\aol^{1\over 2} \RL(\ao) \st \TXL^{1\over 2}\| \leq \|\RL(\ao) \aol\|.
	$$
	Using the spectral theorem and \eqref{eq:GLproper4}, and noting that $\|\ao\| \leq \|\st\|^2 \|\TX\| \leq \kappa^2$ by \eqref{eq:TXbound}, we get
	$$
	\|\aol^{1\over 2} \RL(\ao) \st \TXL^{1\over 2}\| \leq \sup_{u \in [0,\kappa^2]} |\RL(u)(u+\lambda)| \leq \lambda.
	$$
	Using the above inequality and by a direct calculation,
	$$
	\|\TB\|_{\HK} \leq \|\aol^{1\over 2} \RL(\ao) \st \TXL^{1\over 2}\| \|\TXL^{-1\over 2} \TKL^{1\over 2}\| \|\TK^{-{1\over 2}}\EPSRA\|_{\HK} \leq \lambda \DZF^{1\over 2}\|\TK^{-{1\over 2}}\EPSRA\|_{\HK} .
	$$
	Applying Part 2) of Lemma \ref{lem:popFun}, we get
	\be\label{eq:TBbI}
	\|\TB\|_{\HK} \leq \tau R \lambda^{\zeta} \DZF^{1\over 2} .
	\ee
	Applying the above and \eqref{eq:termAI} into \eqref{eq:1ErrTM}, we know that for  any $\zeta\in [0,1],$
	$$
	\|\TXL^{1\over 2}(\ESRA -  \st\EPSRA)\|_{\HK} \leq \tau \DZF^{1\over 2} \left( \DZS +   (2\tau + 1)R \lambda^{\zeta} + \tau R\DZI \lambda^{\zeta-1}  \right).
	$$
	Using the above into \eqref{eq:totErr1}, we can prove the first desired result. 
	\\{\it Case II: $\zeta \geq 1$}\\
	We denote
	\be
	\aoB = \TX^{1\over 2} \st \TX^{1\over 2}, \quad \aoBl = \aoB + \lambda.
	\ee
	Noting that $\ao = \st \TX \st = \st \TX^{1\over 2} (\st \TX^{1\over 2})^*,$ thus following from Lemma \ref{lem:sss} (with $f(u)= (u+\lambda)^{1\over 2} \RL(u)$) and $\st^2 =\st$,
	\begin{align*}
	\|\aol^{1\over 2} \RL(\ao) \st \TX^{\zeta-{1\over 2}} \| = \|\aol^{1\over 2} \RL(\ao) (\st \TX^{1\over 2}) \TX^{\zeta-1} \|  = \| (\st \TX^{1\over 2}) \aoBl^{1\over 2} \RL(\aoB)  \TX^{\zeta-1} \|. 
	\end{align*}
	Adding and subtracting with the same term, using the triangle inequality, 
	\begin{align*}
	\|\aol^{1\over 2}\RL(\ao) \st \TX^{\zeta-{1\over 2}} \| \leq &  \| \st \TX^{1\over 2} \aoBl^{1\over 2} \RL(\aoB)  \aoB^{\zeta-1} \| + \| \st \TX^{1\over 2} \aoBl^{1\over 2} \RL(\aoB)  (\TX^{\zeta-1}-\aoB^{\zeta-1}) \|\\
	\leq & \| \st \TX^{1\over 2}  \aoBl^{1\over 2}\RL(\aoB)  \aoB^{\zeta-1} \| + \| \st \TX^{1\over 2} \aoBl^{1\over 2} \RL(\aoB) \| \| \TX^{\zeta-1}-\aoB^{\zeta-1}\|.
	\end{align*}
	Using Lemma \ref{lem:operDiff}, with \eqref{eq:TXbound} and $\|\aoB\| \leq \|\TX\| \leq \kappa^2,$
	\begin{align*}
	\|\aol^{1\over 2} \RL(\ao) \st \TX^{\zeta-{1\over 2}} \|
	\leq  \| \st \TX^{1\over 2}  \aoBl^{1\over 2} \RL(\aoB)  \aoB^{\zeta-1} \| + \| \st \TX^{1\over 2} \aoBl^{1\over 2} \RL(\aoB) \| \kappa^{2(\zeta-2)_+} \|\TX - \aoB\|^{(\zeta-1)\wedge 1}.
	\end{align*}
	Using $\|A\| = \|A^* A\|^{1\over 2}, \st^2 = \st$,
	the spectral theorem, and \eqref{eq:GLproper4}, for any $s \in [1,\tau],$
	\begin{align*}
	\| \st \TX^{1\over 2} \aoBl^{1\over 2}  \RL(\aoB)  \aoB^{s-1} \| = & \| \aoB^{s-1} \RL(\aoB) \aoBl \aoB  \RL(\aoB)  \aoB^{s-1} \|^{1\over 2}  \\
	\leq& \sup_{u \in [0,\kappa^2]} |\RL(u) u^{s - {1\over 2}} (u+\lambda)^{1\over 2}| \leq \lambda^{s},
	\end{align*}
	and thus we get
	\begin{align*}
	\|\aol^{{1\over 2}-a} \RL(\ao) \st \TX^{\zeta-{1\over 2}} \|
	\leq  \lambda^{\zeta} + \lambda \kappa^{2(\zeta-2)_+} \|\TX - \aoB\|^{(\zeta-1)\wedge 1}.
	\end{align*}
	Using Lemma \ref{lem:proje2te}, $(I -\st)^2 = I -\st$ and $\|A^* A\|=\|A\|^2$, we have 
	$$
	\|\TX - \aoB\| = \|\TX^{1\over 2}(I - \st) \TX^{1\over2}\| \leq\|\TX - \TK\| + \|\TK^{1\over 2}(I - \st) \TK^{{1\over 2}}\| \leq  \DZT + \DZI,
	$$
	and we thus get
	\begin{align}\label{eq:4}
	\|\aol^{1\over 2} \RL(\ao) \st \TX^{\zeta-{1\over 2}} \|
	\leq  \lambda^{\zeta} + \lambda \kappa^{2(\zeta-2)_+}(\DZT + \DZI)^{(\zeta-1)\wedge 1}.
	\end{align}
	Now we are ready to estimate $\|\TB\|_{\HK}.$ 
	By some direct calculations and Part 2) of Lemma \ref{lem:popFun}, 
	\begin{align*}
	\|\TB\|_{\HK} \leq \|\aol^{1\over 2} \RL(\ao) \st \TK^{\zeta-{1\over 2}} \| \|\TK^{{1\over 2} - \zeta}\EPSRA\|_{\HK} \leq \|\aol^{1\over 2} \RL(\ao) \st \TK^{\zeta-{1\over 2}} \| \tau R.
	\end{align*}
	Adding and subtracting with the same term, and using the triangle inequality,
	\begin{align*}
	\|\TB\|_{\HK} 
	\leq  \tau  R \left( \|\aol^{1\over 2} \RL(\ao) \st \TX^{\zeta-{1\over 2}} \| +  \|\aol^{1\over 2} \RL(\ao)\|  \| \TK^{\zeta-{1\over 2}} - \TX^{\zeta-{1\over 2}}  \|\right).
	\end{align*}
	Using the spectral theorem, with $\|\ao\|\leq \|\TX\| \leq \kappa^2$ by \eqref{eq:TXbound} and \eqref{eq:GLproper4}, 
	$$
	\|\aol^{1\over 2} \RL(\ao)\|  = \sup_{u \in ]0,\kappa^2] } |\RL(u) (u+\lambda)^{1\over 2}| \leq \lambda^{1\over 2},
	$$
	and we thus get 
	\begin{align*}
	\|\TB\|_{\HK} 
	\leq   \tau R \left( \|\aol^{1\over 2} \RL(\ao) \st \TX^{\zeta-{1\over 2}} \| +  \lambda^{1\over 2} \| \TK^{\zeta-{1\over 2}} - \TX^{\zeta-{1\over 2}}  \|\right).
	\end{align*}
	Applying Lemma \ref{lem:operDiff}, with \eqref{eq:TKBound} and \eqref{eq:TXbound},
	\begin{align*}
	\|\TB\|_{\HK} 
	\leq   \tau R \left( \|\aol^{1\over 2} \RL(\ao) \st \TX^{\zeta-{1\over 2}} \| +  \lambda^{1\over 2}  \kappa^{(2\zeta-3)_+} \DZT^{(\zeta-{1\over 2})\wedge 1}\right).
	\end{align*}
	Introducing with \eqref{eq:4}, 
	\begin{align*}
	\|\TB\|_{\HK} 
	\leq   \tau R \left( \lambda^{\zeta} + \kappa^{2(\zeta-2)_+} \lambda (\DZT + \DZI)^{(\zeta-1)\wedge 1} +   \kappa^{(2\zeta-3)_+} \lambda^{{1\over 2}}  \DZT^{(\zeta-{1\over 2})\wedge 1}\right) .
	\end{align*}
	Introducing the above inequality and \eqref{eq:termAII} into \eqref{eq:1ErrTM}, noting that $\DZF\geq 1$ and $\kappa^2 \geq 1$, we know that for any $\zeta \geq 1,$
	\begin{align*}
		\|\TXL^{1\over 2}(\ESRA -  \st\EPSRA)\|_{\HK}  \leq 
\tau \DZF^{1\over 2} \left( \DZS +   2R \lambda^{\zeta}+ \kappa^{2(\zeta-1)} R ( \kappa \tau \DZN   + \tau \DZI + \lambda (\DZT + \DZI)^{(\zeta-1)\wedge 1} + \lambda^{{1\over 2}}  \DZT^{(\zeta-{1\over 2})\wedge 1} )  \right).
	\end{align*}
	Using the above into \eqref{eq:totErr1}, and by a simple calculation, we can prove the second desired result.

\subsection{Proof of Lemma \ref{lem:OperDiffProd}}

	Let $\SX = U \Sigma V^* $ be the singular value decomposition of
	$\SX$, where $V:  \mR^{r} \to \HK,$ $U \in \mR^{n \times r}$ and $\Sigma = \mbox{diag}(\sigma_1,\sigma_2,\cdots, \sigma_r)$ with $V^* V = I_{r}$, $U^*U = I_r$ and $\sigma_1 \geq \sigma_2,\cdots, \sigma_r>0.$
	In fact, we can write $V = [v_1, \cdots, v_r]$ with
	$$
	V \ba = \sum_{i=1}^r \ba(i)v_i, \quad \forall \ba \in \mR^r,
	$$
	with $v_i \in \HK$ such that $\la v_i, v_j \ra_{\HK} = 0$ if $i \neq j$ and $\la v_i, v_i \ra_{\HK} =1 $. Similarly, we write $U = [u_1,\cdots, u_r]$, and 
	$$
	\SX = \sum_{i=1}^r \sigma_i \la v_i, \cdot \ra_{\HK}  u_i = \sum_{i=1}^r \sigma_i u_i \otimes v_i.
	$$
	For any $\mu\geq 0$, we decompose $\SX$ as $\mcS_{1,\mu} + \mcS_{2,\mu}$ with 
	$$
	\mcS_{1,\mu} = \sum_{\sigma_i >\mu }\sigma_i u_i \otimes v_i, \quad  \mcS_{2,\mu} = \sum_{\sigma_i \leq \mu  }\sigma_i u_i \otimes v_i,
	$$
	and we will drop $\mu$ to
	write $\mcS_{j,\mu}$ as $\mcS_j$ when it is clear in the text.
	Denote $d$ the cardinality of $\{\sigma_i: \sigma_i>\mu\}$.
	Correspondingly,
	\be\label{eq:S1}
	\mcS_1 = U_1 \Sigma_1 V_1^*, \quad \mcS_2 = U_2 \Sigma_2 V_2^*, 
	\ee
	where $V_1 = [v_1, \cdots, v_{d}]$, $V_2 = [v_{d+1} ,\cdots,v_r],$ $U_1 = [u_1,\cdots, u_d]$,
	$
	U_2 = [u_{d+1},\cdots,u_r],
	$ 
	$
	\Sigma_1 = \mbox{diag}(\sigma_1,\cdots, \sigma_d),
	$
	and $\Sigma_2 = \mbox{diag}(\sigma_{d+1},\cdots,d_r).$
	As the range of $\proj$ is $range(\SX^* \skt^*)$, we can let 
	$$
	\proj = \proj_1 + \proj_2,
	$$
	where $\proj_1$ and $\proj_2$ are projection operators on $range(S_1^* \skt^*)$ and  $range(S_2^* \skt^*)$, respectively.

	As $$\TX = \SX^* \SX =  (U \Sigma V^*)^* U \Sigma V^* = V \Sigma^2 V^* ,$$ we have
	$$
	\|(I - \proj) \TX^{1\over 2}\| =  \|(I - \proj) V \Sigma V^*\| = \|(I - \proj_1 - \proj_2) \sum_{i=1}^2V_i \Sigma_i V_i^*\|.
	$$ 
	As $\proj_1$  is a projection operator on $range(S_1^* \skt^*) (\subseteq range(V_1))$ and  $range(S_1^* \skt^*) (\subseteq range(V_2))$, and $V_1^* V_2 = \bf{0},$ we know that
	$\proj_i V_j = 0$ when $i\neq j$.
	Thus, it follows that
	\begin{align*}
	\|(I - \proj) \TX^{1\over 2}\| = & \|\sum_{i=1}^2(I - \proj_i) (V_i \Sigma_i V_i^*)\| \\
	\leq& \sum_{i=1}^2\|(I - \proj_i) (V_i \Sigma_i V_i^*)\| \\
	\leq& \|(I - \proj_1) (V_1 \Sigma_1 V_1^*)\| + \|I - \proj_2\| \|V_2\|  \|\Sigma_2\|  \|V_2^*\|.
	\end{align*}
	As $\Sigma_2 = diag(\sigma_{d+1},\cdots, \sigma_r)$ with $\sigma_r\leq ,\cdots, \sigma_{d+1} \leq \mu,$ we get
	\be\label{eq:s5}
	\|(I - \proj) \TX^{1\over 2}\|  \leq  \|(I - \proj_1) (V_1 \Sigma_1 V_1^*)\| + \mu.
	\ee
	As $\proj_1$ is the projection operator on $range(S_1^* \skt^*)$, letting $W = \skt S_1$ and for any $\lambda>0$,
	$$\proj_1 = W^* (W W^*)^{\dagger} W \succeq  W^* (W W^* + \lambda I)^{-1} W =  W^* W (W^*W + \lambda I)^{-1},
	$$
	and thus
	$$
	I - \proj_1 \preceq I -   W^* W (W^*W + \lambda I)^{-1} = \lambda (W^*W + \lambda I)^{-1}.
	$$
	It thus follows that
	$$
	T_1^{1\over 2} (I - \proj_1) T_1^{1\over 2} \preceq \lambda T_1^{1\over 2} (W^*W + \lambda I)^{-1} T_1^{1\over 2},
	$$
	where for notational simplicity, we write 
	\be \label{eq:T1}T_1 =  (V_1 \Sigma_1 V_1^*)^2.\ee
	Combing with 
	$$\| (I - \proj) T_1^{1\over 2}\|^2 = \| T_1^{1\over 2}(I - \proj)^2 T_1^{1\over 2}\| =
	\| T_1^{1\over 2} (I - \proj) T_1^{1\over 2}\| ,
	$$
	we know that
	$$
	\| (I - \proj) T_1^{1\over 2}\|^2  \preceq \lambda \| T_1^{1\over 2} (W^*W + \lambda I)^{-1}T_1^{1\over 2} \| \leq  \lambda \| T_{1\lambda}^{1\over 2} (W^*W + \lambda I)^{-1} T_{1\lambda}^{1\over 2}\|.
	$$
	As 
	$$
	T_{1\lambda}^{1\over 2} (W^*W + \lambda I)^{-1}T_{1\lambda}^{1\over 2} = \left( T_{1\lambda}^{-{1\over 2}} (W^*W + \lambda I)T_{1\lambda}^{- {1\over 2}} \right)^{-1} =
	\left(I -  T_{1\lambda}^{-{1\over 2}}(T_1 - W^*W )T_{1\lambda}^{-{1\over 2}} \right)^{-1}, 
	$$
	and if 
	\be\label{eq:s4}
	\|T_{1\lambda}^{-{1\over 2}}(T_1 - W^*W )T_{1\lambda}^{-{1\over 2}}\|  \leq c<1,
	\ee
	then according to Neumann series,
	\be\label{eq:s8}
	\| (I - \proj) T_1^{1\over 2}\|^2  \preceq \lambda \| T_{1\lambda}^{-{1\over 2}} (W^*W + \lambda I)^{-1}T_{1\lambda}^{-{1\over 2}} \| \leq (1 - c)^{-1}\lambda.
	\ee
	If we	choose $\mu = \sqrt{\lambda},$ and introduce the above with $c={1\over 2}$  into \eqref{eq:s5}, one can get 
	\be
	\|(I - \proj) \TX^{1\over 2}\|^2  \leq  (\sqrt{2} + 1)^2\lambda \leq 6\lambda,
	\ee
	which leads to the desired bound.
	
	In what follows, we show that	\eqref{eq:s4}  with $c={1\over 2}$  holds with high probability under the constraint \eqref{eq:subsamLev}. Recall \eqref{eq:T1} and that $W = \skt S_1$ with $S_1$ given by \eqref{eq:S1}.
	Thus, 
	$
	T_1 =V_1 \Sigma_1 V_1^* V_1 \Sigma_1 V_1^* = V_1 \Sigma_1^2 V_1^*,
	$
	and 
	$$
	W^* W =  S_1^* \skt^* \skt S_1 = V_1 \Sigma_1 U_1^* \skt^* \skt U_1 \Sigma_1 V_1^*.
	$$
	Therefore, with $V_1^* V_1 = I,$
	\begin{align}
	T_{1\lambda}^{-{1\over 2}}(T_1 - W^*W )T_{1\lambda}^{-{1\over 2}} = &
	V_1 (\Sigma_{1}^2 + \lambda I)^{-1/2} V_1^*
	V_1 \Sigma_1 ( I - U_1^* \skt^* \skt U_1 )\Sigma_1 V_1^*
	V_1 (\Sigma_{1}^2 + \lambda I)^{-1/2} V_1^* \nonumber \\
	=& V_1 (\Sigma_{1}^2 + \lambda I)^{-1/2} \Sigma_1 ( I - U_1^* \skt^* \skt U_1 )\Sigma_1 (\Sigma_{1}^2 + \lambda I)^{-1/2} V_1^*. \label{eq:s7}
	\end{align}
	It follows that
	$$
	\| T_{1\lambda}^{-{1\over 2}}(T_1 - W^*W )T_{1\lambda}^{-{1\over 2}} \| \leq 
	\| V_1\|  \|(\Sigma_{1}^2 + \lambda I)^{-1/2} \Sigma_1\|^2 \|  I - U_1^* \skt^* \skt U_1  \| \| V_1^*\| 
	\leq  \|  I - U_1^* \skt^* \skt U_1  \|.
	$$
	Using $U_1^* U_1 = I$,
	\begin{align*}
	\|  I - U_1^* \skt^* \skt U_1  \| =  &  \|   U_1^*(I - \skt^* \skt) U_1  \|  \\
	= & \max_{\ba \in \mR^d, \|\ba\|_2=1} | \la U_1^*(I - \skt^* \skt) U_1 \ba, \ba  \ra_2 | \\
	= & \max_{\ba \in \mR^d, \|\ba\|_2=1} | \|U_1 \ba\|_2^2  - \| \skt U_1\ba \|_2^2 |.
	\end{align*}
	Based on a standard covering argument \cite{baraniuk2008simple}, we know that 
	$$\max_{\ba \in \mR^d, \|\ba\|_2=1}| \|U_1 \ba\|_2^2  - \| \skt U_1\ba \|_2^2 |\leq {1 \over 2} $$ with probability at least 
	$$
	1 - 2 (72)^d \exp\left(- {m \over 12^2 c_0' \log^{\beta}n} \right) \geq 1-\delta, 
	$$
	provided that
	\be\label{eq:s6}
	m \geq 144 c_0' \log^{\beta} n\left( \log {2 \over \delta}+ 6d\log 2\right).
	\ee
	Note that by \eqref{eq:empEffDim}
	$$
	b_{\gamma}\lambda^{-\gamma} \geq \tr(\TX \TXL^{-1}) = \sum_{i} {\sigma_i^2 \over \sigma_i^2 + \lambda} \geq \sum_{\sigma_i^2 >\lambda} {\sigma_i^2 \over \sigma_i^2 + \lambda}  \geq {d \over 2}.
	$$
	Thus, a stronger condition for \eqref{eq:s6} is
	\eqref{eq:subsamLev}.
	The proof is complete.

	\subsection{Proof of Lemma \ref{lem:proje2te}}
		Since $\proj$ is a projection operator, $(I-\st)^2 = I -\st$. Then
		\begin{align*}
		\| A^{s}( I - \st )A^{t}\| 	= \| A^{s}( I - \st )(I-\st)A^{t}\|  \leq \| A^{s}( I - \st )\| \|(I-\st)A^{t}\| .
		\end{align*}
		Moreover, by Lemma \ref{lem:operProd}, 
		$$
		\| A^{s}( I - \st )\| =  \| A^{{1\over 2} 2s}( I - \st )^{2s}\| \leq \| A^{{1\over 2}}( I - \st )\|^{2s}.
		$$
		Similarly, $\|(I-\st)A^{t}\| \leq \| ( I - \st )A^{{1\over 2}}\|^{2t}$. Thus,
		\begin{align*}
		\| A^{s}( I - \st )A^{t}\|  \leq \| A^{1\over 2}( I - \st )\|^{2s} \|(I-\st)A^{1\over 2}\|^{2t} =  \|(I-\st)A^{1\over 2}\|^{2(t+s)} .
		\end{align*}
		Using $\|D\|^2 = \|D^*D\|,$
		\begin{align*}
		\| A^{s}( I - \st )A^{t}\|  \leq   \|(I-\st)A (I -\st)\|^{t+s} .
		\end{align*}
		Adding and subtracting with the same term, using the triangle inequality, and noting that $\|I -\st\|\leq 1$ and $s+t\leq 1,$
		\begin{align*}
		\| A^{s}( I - \st )A^{t}\|  \leq&   \|(I-\st)A (I -\st)\|^{t+s} \\
		\leq &\left(\|(I-\st)(A-B) (I -\st)\| + \|(I-\st)B (I -\st)\| \right)^{t+s} \\
		\leq& 
		\|A-B\|^{s+t} + \|(I-\st)B (I -\st)\|^{s+t} ,
		\end{align*}
		which leads to the desired result using $\|D^*D\| = \|DD^*\|$.

\subsection{Proof of Lemma \ref{lem:empEffdim}}
To prove the result, we need the following concentration inequality.
\begin{lemma}
	\label{lem:Bernstein}
	Let $w_1,\cdots,w_m$ be i.i.d random variables in a separable Hilbert space with norm $\|\cdot\|$. Suppose that
	there are two positive constants $B$ and $\sigma^2$ such that
	\be\label{bernsteinCondition}
	\mE [\|w_1 - \mE[w_1]\|^l] \leq {1 \over 2} l! B^{l-2} \sigma^2, \quad \forall l \geq 2.
	\ee
	Then for any $0< \delta <1/2$, the following holds with probability at least $1-\delta$,
	$$ \left\| {1 \over m} \sum_{k=1}^m w_m - \mE[w_1] \right\| \leq 2\left( {B \over m} + {\sigma \over \sqrt{ m }} \right) \log {2 \over \delta} .$$
	In particular, \eref{bernsteinCondition} holds if
	\be\label{bernsteinConditionB}
	\|w_1\| \leq B/2 \ \mbox{ a.s.}, \quad \mbox{and } \quad \mE [\|w_1\|^2] \leq \sigma^2.
	\ee
\end{lemma}
The above lemma is a reformulation of the concentration inequality for sums of Hilbert-space-valued random variables from \cite{pinelis1986remarks}. We refer to  \cite{smale2007learning,caponnetto2007optimal} for the detailed proof.

\begin{proof}[Proof of Lemma \ref{lem:empEffdim}]
	We first use Lemma \ref{lem:Bernstein} to estimate $\tr(\TKL^{-{1\over 2}} (\TX - \TK) \TKL^{-{1\over 2}}).$ Note that
	\bea
	\tr(\TKL^{-{1\over2}}\TX \TKL^{-{1\over 2}}) = {1\over n} \sum_{j=1}^n \|\TKL^{-{1\over 2}} x_j\|_{\HK}^2 = {1\over n} \sum_{j=1}^n \xi_j,
	\eea
	where we
	let $\xi_j = \|\TKL^{-{1\over 2}} x_j\|_{\HK}^2$ for all $j \in [n].$ Besides, it is easy to see that 
	$$
	\tr(\TKL^{-{1\over 2}} (\TX - \TK) \TKL^{-{1\over 2}}) = {1\over n}\sum_{j=1}^n (\xi_j - \mE[\xi_j]).
	$$
	Using Assumption \eqref{boundedKernel},
	$$
	\xi_1 \leq {1\over \lambda} \|x_1\|_{\HK}^2 \leq {\kappa^2 \over \lambda},
	$$
	and
	$$
	\mE[\|\xi_1\|^2] \leq {\kappa^2 \over \lambda} \mE\|\TKL^{-{1\over 2}} x_1\|_{\HK}^2\leq {\kappa^2 \mcN(\lambda) \over \lambda}.
	$$
	Applying Lemma \ref{lem:Bernstein}, we get that there exists a subset $V_1$ of $Z^n$ with measure at least $1-\delta$, such that for all $\bz \in V_1$,
	$$
	\tr(\TKL^{-{1\over2}}(\TX - \TK) \TKL^{-{1\over 2}}) \leq 2\left( {2\kappa^2 \over n\lambda} + \sqrt{\kappa^2 \mcN(\lambda) \over  n\lambda} \right) \log {2 \over \delta} .
	$$
	Combining with Lemma \ref{lem:operDifRes}, taking the union bounds, rescaling $\delta$, and noting that
	\begin{align*}
	\tr(\TXL^{-1}\TX) = & \tr(\TXL^{-{1\over2 }}\TKL^{1\over 2}\TKL^{-{1\over 2}} \TX \TKL^{-{1\over 2}}\TKL^{{1\over 2}}\TXL^{-{1\over 2}}) \\ 
	\leq& \|\TKL^{{1\over 2}}\TXL^{-{1\over 2}}\|^2
	\tr(\TKL^{-{1\over 2}} \TX \TKL^{-{1\over 2}}) \\
	= & \|\TKL^{{1\over 2}}\TXL^{-{1\over 2}}\|^2 \left( \tr(\TKL^{-{1\over 2}} (\TX - \TK) \TKL^{-{1\over 2}} ) + \mcN(\lambda) \right).
	\end{align*}
	we get that there exists a subset $V$ of $Z^n$ with measure at least $1-\delta$, such that for all $\bz \in V$, 
	\bea
	\tr ((\TX+\lambda)^{-1}\TX) \leq 3 a_{n,\delta/2,\gamma}(\theta) \left(2\left( {2\kappa^2 \over n\lambda} + \sqrt{\kappa^2 \mcN(\lambda) \over  n\lambda} \right) \log {4 \over \delta}   +\mcN(\lambda) \right),
	\eea
	which leads to the desired result using $\lambda \leq 1$, $n\lambda\geq 1$ and Assumption \ref{as:eigenvalues}.
\end{proof}

\subsection{Proof for Corollary \ref{cor:nyReg}}
\begin{proof}
	Using a similar argument as that for \eqref{eq:s8}, with $W = \SXS,$ where $\tilde{\bx} = \{x_1,\cdots,x_m\}$, we get for any $\eta>0,$
	$$
	\|(I - \proj)\TK^{1\over 2}\|^2 \leq  \eta \|(\TXS+\eta)^{-1/2}(\TK+\eta)^{1/2}\|^2.
	$$
	Letting $\eta = {1\over m},$ and using Lemma \ref{lem:operDifRes},
	we get that with probability at least $1-\delta,$
	$$
	\|(I - \proj)\TK^{1\over 2}\|^2 \lesssim  {1 \over m} {\log {3m^{\gamma}\over \delta} }.
	$$
	Combining with Corollary \ref{cor:2}, one can prove the desired result.
\end{proof} 

\subsection{Proof of Corollary \ref{cor:ALSnyReg}} 

We first note that in an $L$-ALS Nystr\"{o}m subsampling regime, $\HKS$ can be rewritten as $\HKS =  \overline{range\{\SX^* \skt^\top\}},$ where each row ${1 \over \sqrt{m}}\ba_j^{\top}$ of $\skt$  is i.i.d. drawn according to
$$
\mP \left(\ba = {1\over \sqrt{q_i}} \mathbf{e}_i \right) = q_i, 
$$
Here $\{ \mathbf{e}_i : i \in[n] \}$ is the standard basis of $\mR^n$ and 
$$q_i: = q_{i}(\lambda)  =  {\hat{l}_{i}(\lambda)  \over \sum_j \hat{l}_{j}(\lambda) }.$$

We first introduce the following basic probabilistic estimate.

\begin{lemma}
	\label{lem:concentrSelfAdjoint}
	Let $\mcX_1, \cdots, \mcX_m$ be a sequence of independently and identically distributed self-adjoint Hilbert-Schmidt operators on a separable Hilbert space.
	Assume that $\mE [\mcX_1] = 0,$ and $\|\mcX_1\| \leq B$ almost surely for some $B>0$. Let $\mathcal{V}$ be a positive trace-class operator such that $\mE[\mcX_1^2] \preccurlyeq \mathcal{V}.$
	Then with probability at least $1-\delta,$ ($\delta \in ]0,1[$), there holds
	\bea
	\left\| {1 \over m} \sum_{i=1}^m \mcX_i \right\| \leq {2B \beta \over 3m} + \sqrt{2\|\mathcal{V}\|\beta \over m }, \qquad \beta = \log {4 \tr \mathcal{V} \over \|\mathcal{V}\|\delta}.
	\eea
\end{lemma}
The above lemma was first proved in \cite{hsu2014random,tropp2012user} 
for the matrix case, and it was  later extended to the general  operator case in \cite{minsker2011some}, see also \cite{rudi2015less,bach2015equivalence,dicker2017kernel}.
We refer to \cite{rudi2015less,dicker2017kernel} for the proof.

Using the above lemma, and with a similar argument as that for Lemma \ref{lem:OperDiffProd}, we can estimate the empirical version of the projection error as follows.
\begin{lemma}\label{lem:OperDiffProdALS}
	Let $0<\delta <1$ and $\theta \in[0,1].$
	Given a fix input subset $\bx \subseteq \HK^n$, assume that for $\lambda \in [0,1]$,
	\eqref{eq:empEffDim}
	holds for some $b_{\gamma}>0$, $\gamma \in[0,1]$.
	Then there exists a subset $U_{\bx} $ of $\mR^{m\times n}$ with measure at least $1-\delta$, such that for all $\skt \in U_{\bx}$, 
	\be\label{eq:o3}
	\|(I - \proj)\TX^{1\over 2}\|^2 \leq  3\lambda,
	\ee
	provided that
	\be\label{eq:subsamLevLS}
	m \geq 8 b_{\gamma} \lambda^{-\gamma} L^2 \log {8 b_{\gamma} \lambda^{-\gamma} \over  \delta}.
	\ee	
\end{lemma}

\begin{proof}
	If we choose $u= 0$ in the proof of Lemma \ref{lem:OperDiffProd}, then $\SX = \mcS_1$ and $\mcS_2 = 0$. Similarly, $\TX = T_1$.
	In this case, \eqref{eq:s7} reads as
	\begin{align*}
	\TXL^{-{1\over 2}}(\TX - W^*W )\TXL^{-{1\over 2}} 
	=& V (\Sigma^2 + \lambda I)^{-1/2} \Sigma  ( I - U^* \skt^* \skt U) \Sigma (\Sigma^2 + \lambda I)^{-1/2} V^*. 
	\end{align*}
	Thus, using $V^* V = I$, $U^*U = I$ and $U$ is of full column rank,
	\begin{align*}
	\|\TXL^{-{1\over 2}}(\TX - W^*W )\TXL^{-{1\over 2}} \|
	\leq & \|V\| \|U^*U (\Sigma^2 + \lambda I)^{-1/2} \Sigma U^* ( I -  \skt^* \skt)U\Sigma (\Sigma^2 + \lambda I)^{-1/2} U^*U \| \\
	\leq & \| U (\Sigma^2 + \lambda I)^{-1/2} \Sigma U^* ( I -  \skt^* \skt)U\Sigma (\Sigma^2 + \lambda I)^{-1/2} U^*\|. 
	\end{align*}
	Using $\bK: = \bK_{ \bx \bx} = \SX \SX^* = U \Sigma^2 U^*,$ we get
	\begin{align*}
	\|\TXL^{-{1\over 2}}(\TX - W^*W )\TXL^{-{1\over 2}} \|
	\leq & \| \left( \bK(\bK + \lambda I)^{-1}\right)^{1/2}( I -  \skt^* \skt)\left( \bK(\bK + \lambda I)^{-1}\right)^{1/2}\|. 
	\end{align*}
	Letting $\mcX_i = \left( \bK(\bK + \lambda I)^{-1}\right)^{1/2} \ba_i \ba_i^{*}  \left( \bK(\bK + \lambda I)^{-1}\right)^{1/2}$, it is easy to prove that $\mE [\ba_i \ba_i^{*}] = I,$ according to the definition of ALS Nystr\"{o}m subsampling. Then
	the above inequality can be written as
	\begin{align*}
	\|\TXL^{-{1\over 2}}(\TX - W^*W )\TXL^{-{1\over 2}} \|
	\leq & \| {1\over m} \sum_{i=1}^m (\mE[\mcX_i] - \mcX_i )\|. 
	\end{align*}
	A simple calculation shows that 
	\begin{align*}
	\| \mcX_i\| =&  \ba_i^{*}  \left( \bK(\bK + \lambda I)^{-1}\right)\ba_i  \leq \max_{j \in [n]} {\left( \bK(\bK + \lambda I)^{-1}\right)_{jj} \over q_{j} }\\
	=&
	\max_{j \in [n]}  {l_j(\lambda) \over  q_j}  =  \max_{j \in [n]}	{l_j(\lambda) \sum_k \hat{l}_{k}(\lambda)   \over  \hat{l}_{j}(\lambda) }  \leq L^2 \sum_j {l}_{j}(\lambda) = L^2 \tr(\bK \bK_\lambda^{-1}),
	\end{align*}
	and 
	$$
	\mE[\mcX_i^2] = \mE[\ba_i^{*}  \left( \bK(\bK + \lambda I)^{-1}\right)\ba_i \mcX_i] \leq L^2 \tr(\bK \bK_\lambda^{-1}) \mE[\mcX_i] = L^2 \tr(\bK \bK_\lambda^{-1}) \bK\bK_\lambda^{-1}.
	$$
	Thus, 
	$$
	\| \mE[\mcX_i] - \mcX_i \|   \leq   \mE\| \mcX_i\| +  \|\mcX_i \| \leq 2 L^2 \tr(\bK\bK_\lambda^{-1}), 
	$$
	and 
	$$
	\mE\Big[\big(\mcX_i - \mE[\mcX_i]\big)^2\Big]  \preceq \mE[\mcX_i^2]  \preceq  L^2 \tr(\bK \bK_\lambda^{-1}) \bK\bK_\lambda^{-1}. 
	$$
	Letting $\mcV = L^2 \tr(\bK \bK_\lambda^{-1}) \bK\bK_\lambda^{-1} ,$ we have 
	$$\|\mcV\| \leq L^2 \tr(\bK \bK_\lambda^{-1}) ,
	$$
	and 
	$$
	{\tr(\mcV) \over \|\mcV\|} =   {\tr(\bK \bK_\lambda^{-1}) \over \| \bK \bK_\lambda^{-1}\|} = {\tr(\bK \bK_\lambda^{-1}) \left( 1 + {\lambda \over \|\bK\|}\right)}.
	$$
	Applying Lemma \ref{lem:concentrSelfAdjoint}, noting that $\tr(\bK \bK_\lambda^{-1}) = \tr(\TX \TXL^{-1})$ and $\|\bK\| = \|\TX\|$ as $\TX = \SX^*\SX$, we get that there exists a subset $U_\bx \in \mR^{m \times n}$ with measure at least $1 -\delta$ such that for all $\skt \in U_\bx,$
	\begin{align*}
	\|\TXL^{-{1\over 2}}(\TX - W^*W )\TXL^{-{1\over 2}} \|
	\leq {4 L^2 \tr(\TX \TXL^{-1}) \beta \over 3 m} + \sqrt{2 L^2 \tr(\TX \TXL^{-1}) \beta \over m }, \quad  \beta = \log {4 \tr(\TX \TXL^{-1})(1 + \lambda /\|\TX\|) \over \delta}.
	\end{align*}
	If $\lambda \leq \|\TX\|,$ 	 using Condition \eqref{eq:empEffDim},  we have
	$$\beta \leq \log {4 b_{\gamma} \lambda^{-\gamma} (1+ \lambda/\|\TX\|) \over  \delta} \leq \log {8 b_{\gamma} \lambda^{-\gamma} \over  \delta},
	$$
	and, combining with \eqref{eq:subsamLevLS},
	$$
	{4 L^2 \tr(\TX \TXL^{-1}) \beta \over 3 m} + \sqrt{2 L^2 \tr(\TX \TXL^{-1}) \beta \over m } \leq {2\over 3}.
	$$
	Thus, 
	\bea
	\left\| \TXL^{-1/2}(\TK - M)\TXL^{-1/2}  \right\| \leq {2\over 3}, \quad \forall \skt \in U_{\bx}.
	\eea
	Following from \eqref{eq:s4} and \eqref{eq:s8}, one can prove \eqref{eq:o3} for the case $\lambda\leq \|\TX\|$. The proof for the case $\lambda \geq \|\TX\|$ is trivial:
	$$
	\|(I - \proj )\TX^{1\over2}\|^2 \leq \|I - \proj\|^2 \|\TX^{1\over2}\|^2 \leq \|\TX\| \leq \lambda.
	$$
	The proof is complete.
\end{proof}

With the above lemma, and using a similar argument as that for Corollary \ref{cor:3}, we can prove Corollary \ref{cor:ALSnyReg}.
We thus skip it.

\section{Learning with Kernel Methods}\label{app:learning}
Let the input space $\Xi$ be a closed subset of Euclidean space $\mR^d$, the output space $Y \subseteq \mR$. Let $\mu$ be an unknown but fixed Borel probability measure on $\Xi \times Y$. Assume that $\mathbf \{(\xi_i, y_i)\}_{i=1}^m$ are i.i.d. from the distribution  $\mu$. A reproducing kernel $K$ is a symmetric function $K: \Xi
\times \Xi \to \mR$ such that $(K(u_i, u_j))_{i, j=1}^\ell$ is
positive semidefinite for any finite set of points
$\{u_i\}_{i=1}^\ell$ in $\Xi$. The kernel $K$ defines a reproducing
kernel Hilbert space (RKHS) $(\mathcal{H}_K, \|\cdot\|_K)$ as the
completion of the linear span of the set $\{K_{\xi}(\cdot):=K(\xi,\cdot):
\xi\in \Xi\}$ with respect to the inner product $\la K_{\xi},
K_u\ra_{K}:=K(\xi,u).$ For any $f \in \mathcal{H}_K$, the reproducing property holds: $f(\xi) = \la K_{\xi}, f\ra_K.$

\begin{Exa}
	[Sobolev Spaces]
	Let $X=[0,1]$ and the kernel
	$$
	K(x,x') =
	\begin{cases}
	(1-y)x, & x\leq y; \\
	(1-x)y, & x \geq y.
	\end{cases}
	$$
	Then the kernel induces a Sobolev Space $\HK = \{f : X \to \mR | f \mbox{ is absolutely continuous }, f(0) = f(1) =
	0, f \in L^2(X)\}. $
\end{Exa}
In learning with kernel methods, one considers the following minimization problem
$$ \inf_{f\in \mathcal{H}_K} \int_{\Xi \times Y} (f(\xi) - y)^2 d\mu(\xi,y).$$
Since $f(\xi) = \la K_{\xi},f\ra_K $ by the reproducing property, the above can be rewritten as
$$ \inf_{f\in \mathcal{H}_K} \int_{\Xi \times Y} (\la f, K_{\xi} \ra_K - y)^2 d\mu(\xi,y).$$
Letting $X = \{K_{\xi}: \xi \in \Xi\}$ and defining another probability measure
$\rho(K_{\xi},y) = \mu(\xi,y)$, the above reduces to the learning setting in Section \ref{sec:learning}.

\end{document}